\documentclass[journal]{IEEEtran}
\usepackage{amsfonts}
\usepackage{cite,graphicx,amsmath,amsthm}
\usepackage{fancyhdr}
\usepackage{dsfont}
\usepackage{array,color}
\usepackage{bm}
\usepackage{float}
\usepackage{algpseudocode}
\usepackage{multirow}
\usepackage{booktabs}
\usepackage{makecell}
\usepackage{hyperref}
\allowdisplaybreaks[4]
\usepackage{amsmath,amssymb, graphicx, epstopdf,cite,enumerate,booktabs, setspace, float,stfloats,bm, multirow, lscape, caption, subcaption, graphbox, soul, color, xcolor, hyperref}
\usepackage{diagbox}
\hypersetup{hidelinks}
\usepackage{cite, amsmath,amssymb, bm}
\usepackage{bbding}
\usepackage{amsthm}
\usepackage[normalem]{ulem}
\usepackage[linesnumbered,ruled]{algorithm2e}
\usepackage{flushend}

\newtheorem{lemma}{Lemma}

\newtheorem{proposition}{Proposition}

\newtheorem{corollary}{Corollary}

\newtheorem{property}{Property}

\newtheorem{remark}{Remark}

\newtheorem{claim}{Claim}

\definecolor{purple2}{HTML}{EEE6FD}
\usepackage[table]{xcolor}
\definecolor{tablemaxbg}{HTML}{EEE6FD}
\newcommand{\maxval}[1]{{\setlength{\fboxsep}{0.6pt}\colorbox{tablemaxbg}{#1}}}

\begin{document}

\title{
Memory in the Sky: Low-Altitude Question Answering with Multi-Agent Memory Aggregation
}

\author{
Chengyang Li, Yujie Wan, Shuai Wang, Kejiang Ye, Weijie Yuan, Boyu Zhou, Yik-Chung Wu,\\Chengzhong Xu,~\emph{Fellow, IEEE}, and Huseyin Arslan,~\emph{Fellow, IEEE}
\vspace{-0.1in}

\thanks{A preliminary version of this work has been accepted by IEEE Global Communications Conference (GLOBECOM) 2026, Macau SAR, China \cite{li2026memory}. 
}
\thanks{
Chengyang Li and Yik-Chung Wu are with The University of Hong Kong, Hong Kong SAR, China.

Shuai Wang and Kejiang Ye are with the Shenzhen Institutes of Advanced Technology, Chinese Academy of Sciences, Shenzhen, China.

Yujie Wan, Weijie Yuan, and Boyu Zhou are with the Southern University of Science and Technology, Shenzhen, China.

Chengzhong Xu is with the University of Macau, Macau SAR, China.

Huseyin Arslan is with the Istanbul Medipol University, Istanbul, Turkey.

Corresponding author: Shuai Wang ({\tt\footnotesize s.wang@siat.ac.cn}). 
}
}
\maketitle

\begin{abstract}
This paper studies low-altitude question answering (LAQA), in which distributed unmanned aerial vehicle (UAV) memories are aggregated at a ground server to answer questions about observations over a long horizon.
Unlike conventional resource allocation based on sensing, communication, control, or computation metrics, LAQA requires an explicit measure of memory value.
We propose a generative adversarial exam (GAE) that uses \emph{forward simulation to evaluate memory retrieval} and \emph{exam scores to quantify memory quality}.
This enables the downstream QA value of candidate memories to be measured and optimized without accessing the internal mechanisms of the black-box captioning, retrieval, and reasoning pipeline.
Building on this metric, we develop a memory-centric (MemCen) framework that jointly selects UAVs and allocates transmit power to maximize memory quality under communication constraints.
In the noise-limited regime, we derive a QoM-aware capped water-filling law that explicitly connects task utility with physical-layer power allocation.
We further develop penalty successive optimization (PSO) and learning to memorize (L2M) solvers.
MemCen achieves QA accuracies of 92.4\% and 84.0\% in CARLA Town04 and Town05 under static and dynamic communication conditions, respectively.
In real-world experiments, MemCen achieves 88.5\% QA accuracy on the panoramic multi-agent system (PMAS) benchmark.
Finally, UAV-to-robot-dog demonstrations further validate the practical utility of the acquired memories for environmental understanding and navigation.
\end{abstract}


\section{Introduction}

A low-altitude wireless network (LAWN) integrates sensing, communication, control, and computing across aerial and terrestrial nodes to support diverse functions in mission-critical scenarios \cite{li2026memory,jiang2025integrated}.
In a LAWN, unmanned aerial vehicles (UAVs), or drones, operate over broad areas and continuously observe buildings, roads, objects, and dynamic events while performing missions such as delivery.
As UAVs may operate for extended periods, e.g., several hours, their accumulated observations provide a rich source of historical information that can be queried over long time horizons.
For instance, a user may ask where a particular object was previously observed \cite{anwar2025remembr}.
This motivates low-altitude question answering (LAQA) \cite{wang2025llm,liu2025goal,zhang2023multistep}, where local memories collected by distributed UAVs are aggregated into a global memory at a ground server to support agentic large language model (LLM) inference. 

\subsection{Related Work}

Conventional QA methods are often constrained by short observation horizons, e.g., several minutes \cite{sermanet2024robovqa}, high token costs \cite{xu2024mobility}, or the lack of explicit spatial information in natural-language outputs, such as returning ``beside the building'' rather than a precise position $[x,y,z]$ \cite{das2018embodied}.
Emerging retrieval-augmented memory methods \cite{anwar2025remembr} employ LLM agents to retrieve task-relevant memories over long horizons through function calls conditioned on textual, spatial, and temporal queries.
However, the amount of memory required grows with spatial and temporal coverage, making standalone ground-robot solutions increasingly inefficient for large-scale outdoor scene understanding.
To address this limitation, LAQA constructs memory over a LAWN, leveraging broader fields of view, higher mobility, and multiple distributed agents to substantially extend memory coverage \cite{wang2025llm,liu2025goal,zhang2023multistep}.

Existing LAWN systems adopt sensing coverage \cite{liu2024coverage,cheng2025development}, communication throughput \cite{ye2025integrated}, control safety \cite{zhang2023multistep,jin2025co}, computation efficiency \cite{wang2026low}, or their combinations \cite{jiang2025integrated} as design objectives. LAQA instead explicitly maximizes memory quality. 
This objective distinguishes LAQA resource allocation from traditional LAWN schemes \cite{liu2024coverage,cheng2025development,ye2025integrated,zhang2023multistep,jin2025co,wang2026low,jiang2025integrated} that overlook differences in how UAV memories contribute to QA accuracy. 
Thus, the methods in 
\cite{liu2024coverage,cheng2025development,ye2025integrated,zhang2023multistep,jin2025co,wang2026low,jiang2025integrated} exhibit degraded QA accuracy for diverse questions spanning large spatial and temporal ranges. 

To prioritize transmissions based on data value, semantic communication (SemCom) \cite{liu2025intelligent,yan2022resource} has been explored. These systems reduce information redundancy while preserving semantic similarity (e.g., cosine similarity between BERT embeddings) between the transmitter and receiver. However, semantic similarity may overlook subtle but critical differences in long memory contexts.
It is an intermediate rather than end-to-end task metric for QA. 
Task-oriented communication \cite{wang2020machine,wang2022edge,shi2023task,wen2023task} enables end-to-end optimization of inference performance within integrated sensing, communication, and computation systems. 
If the tasks are closed-set with well-defined object classes and labels \cite{wang2020machine,wang2022edge,wen2023task,shi2023task}, the objective function can be explicitly formulated. 
In LAQA, however, the task is open-set, and the memory must be constructed without prior knowledge of the questions. This uncertainty complicates end-to-end optimization and leaves the design of a suitable memory-oriented objective for LAQA unresolved. 

\subsection{Contributions of This Work}

We introduce a memory-oriented objective based on the observation that, although specific user questions are unbounded and unpredictable, their \emph{patterns and templates are finite and predictable}.
Accordingly, we employ an LLM as a \emph{proxy questioner} to perform forward simulation of potential QA processes.
Building on this principle, we develop a generative adversarial exam (GAE) pipeline that evaluates quality-of-memory (QoM) by generating exams whose difficulty reflects the informativeness of the underlying memory.
A UAV memory that induces fewer correct answers corresponds to a more challenging exam and thus carries higher information value, whereas the converse indicates lower value.
Therefore, the number of unanswered questions provides an explicit and task-oriented measure of memory utility.
This evaluation makes the downstream QA value of candidate memories measurable and optimizable, even when the captioning, retrieval, and reasoning pipeline is treated as a black box.

We integrate GAE with wireless communication conditions into a unified formulation for end-to-end memory optimization, termed memory-centric (MemCen) resource allocation.
For MemCen, we derive a QoM-aware capped water-filling law that characterizes how memory quality governs UAV activation, transmit-power allocation, and power saturation.
We further develop penalty successive optimization (PSO) and learning to memorize (L2M) as practical optimization and low-latency inference methods, respectively.
To facilitate the design and evaluation of LAQA systems, we develop OpenMAMS, an open-source multi-agent memory system and benchmarking platform.
Finally, we validate LAQA through physical multi-UAV experiments and an embodied memory-reuse demonstration.

Our contributions are summarized as follows:
\begin{itemize}
\item We characterize the relationship between a candidate UAV memory and its utility for future open-set QA. GAE forward-simulates candidate-grounded questions against the current global memory and converts the resulting knowledge gap into a computable QoM score. This score exposes task-level memory utility to the resource optimizer without requiring a differentiable captioning, retrieval, or reasoning pipeline.

\item We formulate MemCen as a joint memory-selection and power-allocation problem that integrates memory utility, payload size, channel gain, interference, and transmit power constraints. We derive a QoM-aware capped water-filling law that characterizes how task-weighted QoM determines UAV activation priority and power allocation until memory transmission saturates. We further develop penalty successive optimization (PSO) and learning to memorize (L2M) as practical optimization and low-latency inference methods, respectively.

\item We implement MemCen in OpenMAMS and evaluate it in Town04 and Town05 under heterogeneous UAVs, non-uniform memory payloads, mobility, and building blockage. We further demonstrate UAV-to-ground-robot memory reuse, showing how aerial memories support downstream environmental understanding and navigation.
\end{itemize}

\subsection{Outline and Notations}

The remainder of this paper is organized as follows.
Section \ref{section2} states the LAQA problem. 
Section \ref{section3} presents the GAE model.
Section~\ref{section4} presents the cross-layer formulation, its
QoM-aware allocation structure, and practical PSO/L2M solvers.
Subsequently, Section~\ref{section5} presents the simulation results.
Section~\ref{section6} reports the real-world experiments.
Finally, Section~\ref{section7} concludes this work.

Italic, lowercase bold, uppercase bold, and calligraphic letters denote scalars, vectors, matrices, and sets, respectively.
The operators $(\cdot)^T$, $(\cdot)^H$, $(\cdot)^{-1}$ take the transpose, conjugate transpose, and inverse of a matrix, respectively.
$\nabla f$ represents the gradient of a function $f$.
$\mathcal{CN}(0,1)$ is the standard circularly symmetric complex Gaussian distribution.
$\mathbb{R}$ and $\mathbb{C}$ denote the real and complex fields, respectively. 
Finally, $[a_1,a_2,\cdots]^{T}$ represents a column vector, $\left\Vert\cdot\right\Vert_p$ represents the $\ell_p$-norm of a vector, 
and $\mathcal{O}(\cdot)$ means the order of arithmetic operations.

\begin{figure*}[t]
    \centering
    \includegraphics[width=1.0\textwidth]{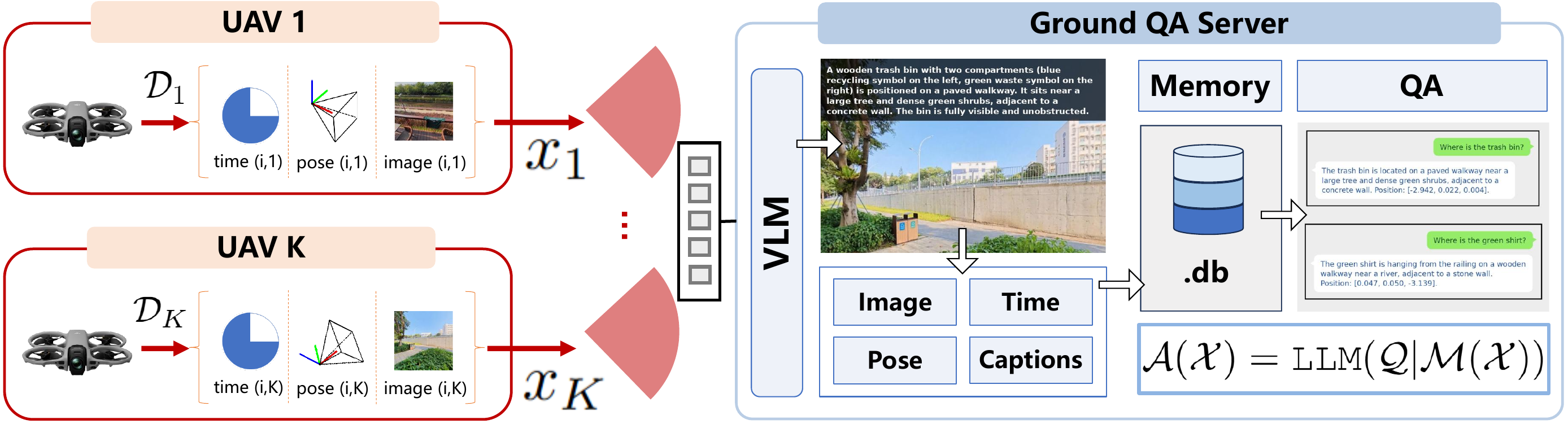}
    \caption{System architecture of LAQA with memory collection and memory-empowered query.}
    \label{fig1}
    \vspace{-0.1in}
\end{figure*}

\section{System Model and Problem Formulation}\label{section2}

We consider the LAQA system illustrated in Fig.~\ref{fig1}, consisting of a ground server and $K$ UAVs.
Let $\mathcal{K}=\{1,\ldots,K\}$ denote the UAV set and $\mathcal{D}=\{\mathcal{D}_k\}_{k\in\mathcal{K}}$ the distributed data collected by the UAVs.
The LAQA system aggregates these data at the ground server to construct a long-horizon global memory $\mathcal{M}$.
Based on $\mathcal{M}$, the system answers a set of spatiotemporal questions $\mathcal{Q}=\{\mathbf{q}_j\}$ and produces the corresponding answers $\mathcal{A}=\{\mathbf{a}_j\}$, where $\mathbf{q}_j$ and $\mathbf{a}_j$ denote the vector representations of the $j$-th question and its corresponding answer, respectively.

\subsection{System Overview}

\subsubsection{Memory Collection}

Each UAV performs simultaneous localization and mapping (SLAM) to obtain a sequence of time-stamped pose-image frames,
$\mathcal{D}_k=\left\{(t_{i,k},\mathbf{s}_{i,k},\mathbf{v}_{i,k})\right\}_{i\in\mathcal{I}_k}$, where $t_{i,k}$ denotes the timestamp, $\mathbf{s}_{i,k}\in\mathbb{R}^{6}$ denotes the 6D pose \cite{anwar2025remembr}, and $\mathbf{v}_{i,k}\in\mathbb{R}^{L\times W\times C}$ denotes the $i$-th image collected by UAV $k$. Here, $L$, $W$, and $C$ denote the image height, width, and number of channels, respectively, with $C=3$ for red-green-blue (RGB) images. The frame index set is defined as
$\mathcal{I}_k=\{1,\ldots,|\mathcal{D}_k|\}$, where $|\mathcal{D}_k|$ denotes the number of frames collected by UAV $k$.
Let $V_{i,k}$ denote the payload size of frame $i$ at UAV $k$. The aggregate candidate payload of UAV $k$ is
$S_k=\sum_{i\in\mathcal I_k}V_{i,k}$, which accommodates heterogeneous frame sizes and different numbers of frames across UAVs. The payload associated with timestamp and pose metadata is negligible under the considered settings.
Due to limited communication resources, the LAWN can upload memories from only a subset of UAVs to the ground server. We therefore introduce a binary memory-selection set $\mathcal{X}=\{x_{1},\cdots, x_{K}\}$ (vector form $\mathbf{x}=[x_{1},\cdots, x_{K}]^T$), where $x_k\in\{0,1\}$ indicates whether UAV $k$ is selected for memory upload.
Equivalently, let $\mathcal{S}(\mathcal{X})=\{k\in\{1,\cdots,K\}:x_{k}=1\}$ 
denote the set of selected UAVs. The resulting memory data collected from the LAWN are then given by
\begin{align}
&\mathcal{E}(\mathcal{X}) =\bigcup_{k\in\mathcal{S}(\mathcal{X})}
\bigcup_{i\in \mathcal{I}_k}
\Big\{(t_{i,k},\mathbf{s}_{i,k},\mathbf{v}_{i,k})\Big\}.
\label{data}
\end{align}

\subsubsection{Memory-Empowered Query}

For each selected UAV, the ground server captions the uploaded images using a VLM,
$\mathbf{c}_{i,k}=\texttt{VLM}\left(\mathbf{v}_{i,k}\right)$,
yielding a set of captions $\mathcal{C}=\{\mathbf{c}_{i,k}:k\in\mathcal{S}(\mathcal{X}), i\in\mathcal{I}_k\}$. These captions summarize the visual observations collected by the UAVs over time and are organized into a queryable memory.
We encode the captions using a text embedding function $E(\cdot)$ and store the resulting time-pose-embedding tuples in a vector database, e.g., \textit{Milvus}. The global memory is given by
\begin{align}\label{memory}
        \mathcal{M}(\mathcal{X})=&
        \mathcal{M}_0\bigcup 
        \bigcup_{k\in\mathcal{S}(\mathcal{X})}
        \mathcal{M}_k,
        \\
        \mathcal{M}_k=&\bigcup_{i\in\mathcal{I}_k}
\left\{\left(t_{i,k},\mathbf{s}_{i,k},E\left(\mathbf{c}_{i,k}\right)\right)\right\},
\end{align}
where $\mathcal{M}_0$ denotes the historical memory available at the ground server before the current memory collection. The resulting vector memory supports efficient retrieval over millions of embeddings \cite{anwar2025remembr}.
Conditioned on $\mathcal{M}(\mathcal{X})$, the system answers each query $\mathbf{q}_j\in\mathcal{Q}$ as
$\mathbf{a}_j(\mathcal{X})
=
\texttt{LLM}\!\left(
\mathbf{q}_j
\mid
\mathcal{M}(\mathcal{X})
\right)$.
Accordingly, the complete QA process can be written compactly as
\begin{align}\label{llm}
    \mathcal{A}(\mathcal{X})=\texttt{LLM}(\mathcal{Q}|\mathcal{M}(\mathcal{X})).
\end{align}

\subsection{Memory Quality Optimization}

In the considered system, the controllable design variables include the UAV memory-selection decisions $\mathcal{X}=\{x_1,\ldots,x_K\}$ and the corresponding transmit powers $\mathcal{P}=\{p_1,\ldots,p_K\}$.
For each selected UAV, i.e., $x_k=1$, its achievable transmission rate must be sufficient to upload the candidate memory payload $S_k$ within the transmission deadline $T$.
Under spatial multiplexing, this requirement can be modeled as \cite{wang2020machine}
\begin{align}
B\log_2\left(
1+\frac{H_k p_k}
{\sum_{j\neq k}^{K} I_{k,j}p_j+\sigma^2}
\right)
\geq
\frac{x_k S_k}{T},
\ \forall k\in\mathcal{K},
\label{rate_constraint}
\end{align}
where $p_k$ denotes the transmit power of UAV $k$, $B$ is the transmission bandwidth, and $\sigma^2$ denotes the noise power.
Moreover, $H_k$ denotes the desired-link channel gain from UAV $k$ to its corresponding receive stream at the ground server, while $I_{k,j}$ ($I_{k,k}=H_k$) denotes the effective interference gain from UAV $j$ to the receive stream associated with UAV $k$ \cite{wang2020machine}.
To limit energy consumption and inter-UAV interference, the transmit powers are subject to both individual and sum-power constraints:
$\{0 \leq p_k \leq P_{\text{max}}, \forall k\}$ and $\sum_{k=1}^Kp_k\leq P_{\mathrm{sum}}$, where $P_{\mathrm{max}}$ and $P_{\mathrm{sum}}$ denote the individual and total transmit-power budgets, respectively \cite{ye2025integrated,wang2020machine,wang2020angle}.

Subject to the communication constraints, our objective is to maximize the answer accuracy of the LAQA system. 
For a question $q$ drawn from the target query distribution, this accuracy is $\Pr(a(q|\mathcal M)=a^*(q))$, where $a(q|\mathcal M)$ and $a^*(q)$ denote the predicted and ground-truth answers, respectively.
Since the answer accuracy is determined by the aggregated vector memory $\mathcal{M}$, we define a memory-quality function $\Psi(\mathcal{M})$ as the resulting LAQA accuracy conditioned on $\mathcal{M}$.
Accordingly, the MemCen problem is formulated as follows:
\begin{subequations}
\begin{align}
\mathsf{P}:~~&\mathop{\mathrm{max}}_{\substack{\mathcal{M},\mathcal{X},\mathcal{P}}}
\quad \Psi\left[\mathcal{M}(\mathcal{X})\right],\\
 \! \!\! \! \textrm{s.t.} ~~ & 
 \mathrm{log}_2\left(1+\frac{H_{k}p_{k}}{\sum_{j\neq k}^KI_{k,j}p_{j}+
\sigma^2} \right)
{\geq \frac{x_{k}S_k}{BT}, \ \forall k,}  \label{Pb} \\
       \! \! \! \!  & 0\leq p_k\leq P_{\mathrm{max}}, \ \forall k,  \ \sum_{k=1}^{K}p_{k} \leq P_{\mathrm{sum}}. \label{Pc} \\
   \! \! \! \!  &  x_k\in\{0,1\}, \ \forall k.
   \label{Pd}
\end{align}
\end{subequations}

Since $\Psi(\mathcal{M})$ differs fundamentally from the classical sum-rate objective
$R(\mathcal{P})=\sum_{k=1}^K B\log_2\left(
1+\frac{H_kp_k}{\sum_{j\neq k}^{K}I_{k,j}p_j+\sigma^2}
\right)$,
there exists an inherent \emph{memorization--communication tradeoff}.
A direct weighted combination of $\Psi$ and $R$ is inappropriate because they have different units and numerical scales.
Let $(\Psi^{\mathrm A},R^{\mathrm A})$ denote the point obtained by maximizing QA accuracy, and let $(\Psi^{\mathrm R},R^{\mathrm R})$ denote that obtained by maximizing the sum rate.
We normalize the two objectives as
\begin{align}
\bar\Psi=\frac{\Psi-\Psi^{\mathrm R}}{\Psi^{\mathrm A}-\Psi^{\mathrm R}}, \
\bar R=\frac{R-R^{\mathrm A}}{R^{\mathrm R}-R^{\mathrm A}}.
\label{eq:normalized-objectives}
\end{align}
Accordingly, the normalized accuracy--rate optimization problem is formulated as
\begin{align}
\mathsf{Q}:~~&\mathop{\mathrm{max}}_{\substack{\mathcal{M},\mathcal{X},\mathcal{P}}}
\ \mu\bar\Psi\left[\mathcal{M}(\mathcal{X})\right] +(1-\mu)\bar R(\mathcal{P}),
\nonumber
\\
 \! \!\! \! \textrm{s.t.} ~~ & \mathrm{constraints} \ \eqref{Pb}, \, \eqref{Pc}, \, \eqref{Pd}.
\end{align}
where $\mu\in(0,1)$ controls the tradeoff between memorization and communication performance.
By varying $\mu$ and repeatedly solving $\mathsf{Q}$, a set of boundary operating points can be obtained, thereby characterizing the achievable accuracy--rate region.

The challenge in solving $\mathsf{P}$ and $\mathsf{Q}$ is that the memory-quality function $\Psi$ has no explicit analytical form.
Section~III therefore introduces GAE to obtain a computable surrogate for $\Psi$.

\emph{Remark 1 (Modeling of ${H_k,I_{k,j}}$)}:
We do not impose a distance-based threshold to discard potential interferers.
For receive stream $k$, all other active UAVs constitute the interference set
$\mathcal J_k=\{j\in\mathcal K:j\neq k,\,p_j>0\}$ and the aggregate interference power is given by $\sum_{j\in\mathcal J_k}I_{k,j}p_j$.
Hence, weak interferers are retained in the model, while their impact is naturally attenuated by the corresponding beamformers.
Let $\mathbf{h}_k\in\mathbb{C}^{N}$ denote the channel vector from UAV $k$ to the $N$-antenna array at the ground server.
Given the receive beamformer $\mathbf{w}_k$, the effective desired-link and interference-link gains are respectively defined as $H_{k}=|\mathbf{w}_k^H\mathbf{h}_{k}|^{2}$ for the desired link and $I_{k,j}=|\mathbf{w}_k^H\mathbf{h}_{j}|^{2}$ for the interference link from UAV $j$ to $k$.

\emph{Remark 2 (MRC and IRC Receivers)}:
We consider two representative choices of the receive beamformer $\mathbf{w}_k$.
The first is the low-complexity maximum-ratio combining (MRC) receiver, $\mathbf{w}_k=
\left\Vert\mathbf{h}_{k}\right\Vert_2^{-1}
\mathbf{h}_{k}$, which requires only the desired-link channel and is therefore attractive for low-complexity reception \cite{wang2020machine}.
However, MRC does not explicitly suppress multiuser interference.
Under MRC, the effective gains reduce to $H_k= \|\mathbf{h}_k\|_2^2$ and 
$I_{k,j}=\frac{|\mathbf{h}_k^H\mathbf{h}_{j}|^2}{\left\Vert\mathbf{h}_{k}\right\Vert_2^2}$. 
The second is interference rejection combining (IRC), which actively suppresses multiuser interference \cite{tusha2024interference}.
Defining the interference-plus-noise covariance matrix for stream $k$ as
$\mathbf{R}_k
=
\sum_{j\neq k}\mathbf{h}_j\mathbf{h}_j^H
+\sigma^2\mathbf{I}_N$,
the IRC beamformer is obtained from
$\mathbf{w}_k
=
\arg\max_{\mathbf{w}}
\frac{
\mathbf{w}^H\mathbf{h}_k\mathbf{h}_k^H\mathbf{w}
}{
\mathbf{w}^H\mathbf{R}_k\mathbf{w}
}$. 
We construct $\mathbf{R}_k$ without transmit-power weighting for receive-beamformer design.
IRC is particularly suitable for interference-limited operation when the ground server has access to multiuser channel information and can estimate the interference-plus-noise covariance matrix.

\section{Proposed Memory Quality Model}\label{section3}

The value of a memory is determined by its contribution to downstream QA rather than by the information it contains alone.
Even a frame containing task-relevant evidence may provide little benefit if the deployed system fails to caption it accurately, retrieve it effectively, or reason over it correctly.
Accordingly, we define memory quality with respect to both the deployed black-box QA pipeline and the current global memory $\mathcal{M}_0$.
Specifically, the utility of a candidate memory $\mathcal{M}_k$ is defined by the marginal improvement in QA performance obtained when $\mathcal{M}_k$ augments $\mathcal{M}_0$.
However, this utility cannot be evaluated directly at collection time because the future user questions are unknown.

\subsection{Generative Adversarial Exam}

GAE estimates this latent utility through a candidate-grounded proxy exam.
Specifically, candidate memory $\mathcal{M}_k$ acts as the examiner and generates a set of questions 
$\widetilde{\mathcal Q}_k=\{\widetilde q_{k,\ell}\}_{\ell=1}^{L_k}$ from verified facts contained in its pilot data, while the current global memory $\mathcal{M}_0$ acts as the examinee.
The normalized QoM is defined as
\begin{equation}
\begin{aligned}
\texttt{QoM}_k
&=
\frac{1}{L_k}
\underbrace{
\sum_{\ell=1}^{L_k}
\mathds{1}\left\{
\mathcal M_0 \ \text{incorrectly answers} \
\widetilde{q}_{k,\ell}
\right\}
}_{:=\texttt{Score}_k}.
\end{aligned}
\label{eq:qom-definition}
\end{equation}
where $\mathds{1}\{\cdot\}$ denotes the indicator function and $\texttt{Score}_k$ denotes the number of questions that cannot be correctly answered using $\mathcal{M}_0$.
Correctness is evaluated through the deployed retrieval-and-answering pipeline.
Hence, a larger $\texttt{QoM}_k$ indicates a larger knowledge gap in the current global memory and, consequently, a higher potential utility of $\mathcal{M}_k$.

To account for nonuniform future queries, we partition their evidence sources as
$\mathcal{Q}=\bigcup_{k=1}^{K}\mathcal{Q}_k$
and define the query prior
\begin{align}
\eta_k=\Pr(q\in\mathcal Q_k),\ \sum_{k=1}^{K}\eta_k=1.
\label{eq:question-prior}
\end{align}
The first-order marginal utility is then approximated by $\eta_k\texttt{QoM}_k$. If a total exam
budget $L$ is allocated according to $L_k\approx L\eta_k$, then
$\texttt{Score}_k=L_k\texttt{QoM}_k\approx L\eta_k\texttt{QoM}_k$. 
This yields the following optimization-compatible approximation:
\begin{align}\label{gae}
\Psi\left[\mathcal{M}(\mathcal{X})\right]
\approx\Psi(\mathcal M_0)+\sum_{k=1}^{K}\eta_k\texttt{QoM}_k x_k.
\end{align}

When no task-specific prior is available, $\eta_k$ can be chosen according to the data-size prior under identical sampling density.
For task-conditioned missions, a semantic belief map $B(i,j)$ can be used to assign
$\eta_k
\propto
\sum_{(i,j)\in\Omega_k} B(i,j)$, where $\Omega_k$ denotes the region observed by UAV $k$.
The server can recompute these weights whenever the mission prior changes.

\begin{figure}[t]
    \centering
    \includegraphics[width=0.49\textwidth]{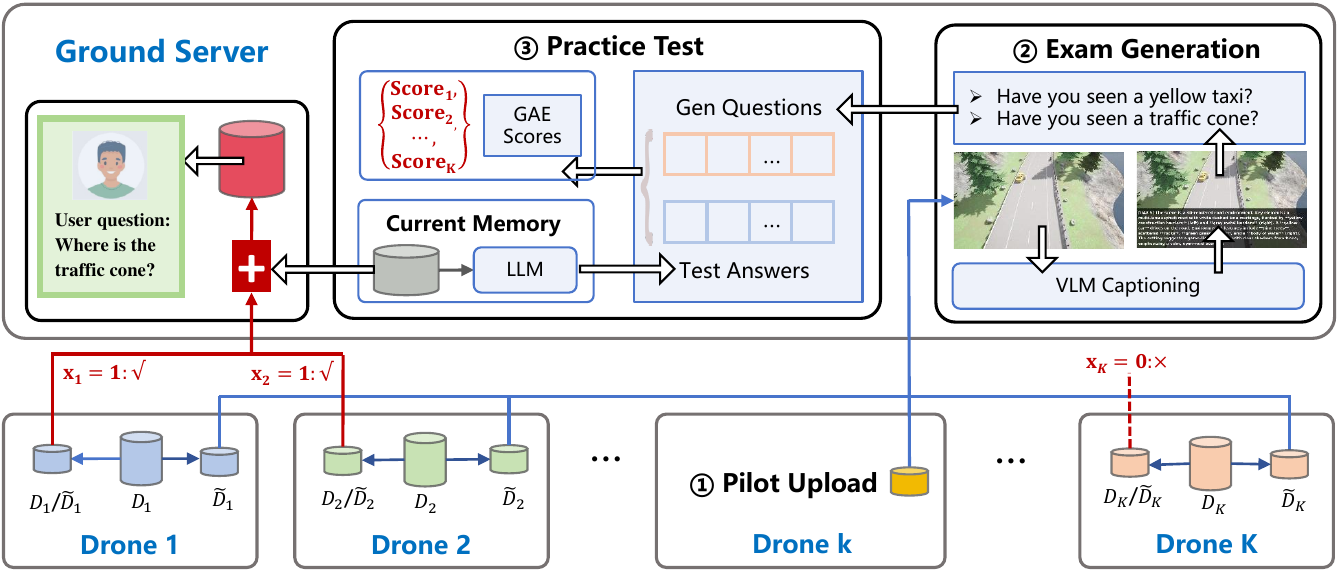}
    \caption{Architecture of GAE for QoM computation.}
    \label{fig2}
    \vspace{-0.1in}
\end{figure}

The architecture of GAE is illustrated in Fig.~\ref{fig2}.
Its inputs are the distributed candidate datasets $\{\mathcal{D}_k\}$ and the pre-collection memory $\mathcal{M}_0$ at the ground server.
The GAE pipeline consists of three steps.
\begin{itemize}
    \item 
    First, each UAV uploads a representative pilot subset $\widetilde{\mathcal{D}}_k\subseteq\mathcal{D}_k$. Random sampling provides a low-complexity fallback, whereas key-frame sampling reduces the risk of overlooking sparse but informative events.

\item Second, the server captions the pilot data and generates $L_k$ grounded question--answer pairs from the resulting observations.

\item Third, the deployed retrieval-and-answering pipeline evaluates whether $\mathcal{M}_0$ can correctly answer these questions and outputs structured answerability decisions. The resulting decisions are deterministically counted to obtain $\texttt{Score}_k$ and $\texttt{QoM}_k$ without prompting.
\end{itemize}

\subsection{Overhead Analysis}\label{sec:gae-overhead}

GAE incurs two sources of overhead:
\begin{align}
\Delta T_{\mathrm{GAE}}=T_{\mathrm{Pilot}}+T_{\mathrm{Comp}},
\label{eq:gae-overhead}
\end{align}
where $T_{\mathrm{Pilot}}$ denotes the wireless latency for uploading pilot data, and $T_{\mathrm{Comp}}$ denotes the server-side latency for captioning, exam generation, retrieval, and answering.
Let $\widetilde{\mathcal I}_k$ denote the pilot-frame index set of UAV $k$, and define the corresponding pilot payload as $\widetilde S_k=\sum_{i\in\widetilde{\mathcal I}_k}V_{i,k}$. Under a feasible equal-power pilot uplink with
$\bar p=\min\{P_{\mathrm{max}},P_{\mathrm{sum}}/K\}$, the common pilot-upload duration is given by
\begin{align}
T_{\mathrm{Pilot}}
=\max_{k\in\mathcal K}
\frac{\widetilde S_k}
{B\log_2\!\left(1+
\frac{H_k\bar p}{\sum_{j\ne k}I_{k,j}\bar p+\sigma^2}\right)}.
\label{eq:pilot-overhead}
\end{align}
Define the pilot ratio as
$\rho_k=|\widetilde{\mathcal D}_k|/|\mathcal D_k|$.
For a small pilot, $\widetilde S_k\ll S_k$, and hence the pilot-upload overhead is substantially lower than that of uploading the complete candidate memory.
Increasing the pilot size may improve the reliability of QoM estimation and ranking, but incurs a larger $T_{\mathrm{Pilot}}$.

For the computational overhead $T_{\mathrm{Comp}}$, directly feeding the entire historical memory $\mathcal M_0$ to the LLM for every proxy question would require repeatedly processing a long context.
Instead, iterative memory retrieval (IMR) \cite{anwar2025remembr} retrieves a compact evidence subset
$\mathcal S_0^*\subseteq\mathcal M_0$ and performs
\begin{align}\label{imr}
\widetilde{\mathcal A}_k
=\texttt{LLM}(\widetilde{\mathcal Q}_k|\mathcal S_0^*),
\ \mathcal S_0^*\subseteq\mathcal M_0.
\end{align}
The retrieval process terminates once sufficient supporting evidence is identified or the maximum number of retrieval rounds is reached.
To ensure reliable scoring, exam generation is conditioned on verified scene facts and produces grounded question--answer pairs in structured JSON format.
The answerability prompt prohibits unsupported guessing and requires evidence identifiers.
Malformed or truncated JSON outputs are treated as invalid runs rather than incorrect answers, preventing formatting failures from biasing the QoM score.

Let $Z_{k,\ell}$ denote the number of retrieval rounds required for the $\ell$-th question of UAV $k$, $T_{\mathrm{gen}}$ the latency of one exam-generation call, and $T_0$ the latency of one retrieval-and-answering round.
With $N_{\mathrm{inf}}$ effective concurrent inference slots, the computational latency can be approximated as
\begin{align}
T_{\mathrm{Comp}}
\approx
T_{\mathrm{cap}}
+
\left\lceil
\frac{K}{N_{\mathrm{inf}}}
\right\rceil
T_{\mathrm{gen}}
+
\left\lceil
\frac{\sum_{k,\ell}Z_{k,\ell}}{N_{\mathrm{inf}}}
\right\rceil
T_0,
\label{eq:gae-computation}
\end{align}
where $T_{\mathrm{cap}}$ denotes the one-time latency for captioning the pilot data.
Accordingly, one GAE evaluation requires
\begin{align}
N_{\mathrm{LLM}}=K+\sum_{k,\ell}Z_{k,\ell}, \
N_{\mathrm{ret}}=\sum_{k,\ell}Z_{k,\ell}.
\label{eq:gae-call-count}
\end{align}
where $N_{\mathrm{LLM}}$ and $N_{\mathrm{ret}}$ denote the numbers of LLM calls and retrieval operations, respectively.
The total computational work scales as $\mathcal O\!\left(K+\sum_{k=1}^{K}\sum_{\ell=1}^{L_k}Z_{k,\ell}\right)$. Concurrency reduces wall-clock latency but does not reduce the total amount of computation.

\section{MemCen Resource Allocation}\label{section4}

Substituting the task-driven marginal utility derived in Section~III renders $\mathsf{P}$ computable:
\begin{subequations}
\begin{align}
\mathsf{P}1:\mathop{\mathrm{max}}_{\substack{\mathcal{X},\mathcal{P}}}
\quad& {\sum_{k=1}^K\eta_k\texttt{QoM}_k x_k},\\
 \! \!\! \! \textrm{s.t.} ~~~~ & 
 \mathrm{log}_2\left(1+\frac{H_{k}p_{k}}{\sum_{j\neq k}^KI_{k,j}p_{j}+
\sigma^2} \right)
\geq \lambda_k x_k, \ \forall k, \label{Pb2} \\
&  \textsf{constraints } (\ref{Pc}),(\ref{Pd}).
\end{align}
\end{subequations}
Compared with constraint~\eqref{Pb}, constraint~\eqref{Pb2} accounts for the pilot data that have already been uploaded during GAE.
Specifically, given the measured pilot payload $\widetilde S_k$ and the common pilot-upload duration $T_{\mathrm{Pilot}}$ in \eqref{eq:pilot-overhead}, the normalized rate requirement for transmitting the remaining candidate memory is
$\lambda_k
=
\frac{S_k-\widetilde S_k}
{(T-T_{\mathrm{Pilot}})B}$.
Thus, only the residual payload $S_k-\widetilde S_k$ needs to be delivered within the remaining wireless transmission budget $T-T_{\mathrm{Pilot}}$.

$\mathsf{P}1$ couples binary selection and continuous power. It
remains nonconvex after relaxing $\mathbf x$ because of multiuser
interference. 
In the following, we adopt PSO to obtain a
stationary solution through convex surrogate subproblems. We also use a learned L2M policy to achieve real-time inference.

\subsection{Penalty Augmentation}

To address the discontinuity in constraint \eqref{Pd}, we relax the binary constraint $x_{k}\in\{0,1\}$ into an affine constraint $0\leq x_{k} \leq 1$, $\forall k$. 
However, the relaxation is generally not tight and the solution to the relaxed problem could be $0<x_{k}<1$. 
To promote a binary solution for the relaxed variable $\{x_{k}\}$, we augment the objective function with a penalty term as in \cite{rinaldi2009new}. 
Here, we adopt the following penalty function \cite{lucidi2010exact}:
$\Xi(\mathcal X)=\frac{1}{\beta}\sum_{k=1}^K x_{k}(1-x_k)$,
where $\beta>0$ is the penalty parameter.
With this penalty, problem $\mathsf{P}1$ is transformed into 
\begin{subequations}
\begin{align}
\mathsf{P}2:\mathop{\mathrm{min}}_{\substack{\mathcal{X},\mathcal{P}}}
\quad& {-\sum_{k=1}^K\eta_k\texttt{QoM}_k x_k}
+
\Xi(\mathcal X)   \label{P2a}
,\\
 \! \!\! \! \textrm{s.t.} ~~~~ & 
{\Theta}_{k}(\mathcal{X},\mathcal{P})\leq 0, \ \forall k, \label{P2b} \\
\! \! \! \!  & 0\leq p_k\leq P_{\mathrm{max}}, \ \forall k,  \ \sum_{k=1}^{K}p_{k} \leq P_{\mathrm{sum}}, 
\label{P2c}
\\
   \! \! \! \!  &  0\leq x_k \leq 1, \ \forall k,
   \label{P2d}
\end{align}
\end{subequations}
where 
\begin{align}
  {\Theta}_{k}(\mathcal{X},\mathcal{P})
  =
  \lambda_k x_k-\log_2 \left( 1 + \frac{H_{k}p_{k}}
	{
    \sum_{j\neq k} I_{k,j}p_{j}+\sigma^2} \right).
\end{align}  
According to \cite[Proposition 1]{lucidi2010exact}, there exists
$\bar\beta>0$ such that $\mathsf{P}1$ and $\mathsf{P}2$ are globally
equivalent for a sufficiently small $0<\beta\leq\bar\beta$.

\subsection{Successive Optimization and Learning to Memorize}

The remaining nonconvex terms are the functions $\Xi(\mathcal X)$ in \eqref{P2a} and ${\Theta}_{k}(\mathcal{X},\mathcal{P})$ in \eqref{P2b}.
We apply successive optimization to these terms by constructing a sequence of upper bounds $\{\widehat{\Xi}\}$ on $\Xi(\mathcal X)$ and replacing $\Xi(\mathcal X)$ in $\mathsf{P}2$ with $\{\widehat{\Xi}\}$ to obtain the surrogate problems. 
Similarly, we construct upper bounds $\{\widehat{\Theta}_k\}$ on 
$\Theta_k$ and replace $\Theta_k$ in $\mathsf{P}2$ with $\{\widehat{\Theta}_k\}$.
Specifically, given any feasible solution $\{\mathcal{X}^\star,\mathcal{P}^\star\}$ to $\mathrm{P}2$, we define two surrogate functions
\begin{subequations}
\begin{align}
&
\quad\quad
\widehat{\Xi}(\mathcal X|\mathcal X^\star ) = 
\sum_{k=1}^K 
\left(
\frac{1}{\beta}x_{k}-\frac{2}{\beta}x_{k}^{\star}x_{k}+\frac{1}{\beta}x_{k}^{\star^2}
\right),
\\
&\widehat{\Theta}_{k}(\mathcal{X},\mathcal{P}|\mathcal{P}^\star)
=\lambda_kx_k-\frac{1}{\mathrm{ln}2}
\Bigg[
\mathrm{ln}\left(\sum_{l=1}^K\frac{I_{k,l}p_{l}}{\sigma^2}+1\right)
\nonumber\\
&
-
\mathrm{ln}\left(\sum_{l=1,l\neq k}^K\frac{I_{k,l}p^\star_{l}}{\sigma^2}+1\right)
-\left(\sum_{l=1,l\neq k}^K\frac{I_{k,l}p^\star_{l}}{\sigma^2}+1\right)^{-1}
\nonumber\\
&
\times\left(\sum_{l=1,l\neq k}^K\frac{I_{k,l}p_{l}}{\sigma^2}+1\right)
+1
\Bigg], 
\label{Phi}
\end{align}
\end{subequations}
and the following proposition can be established.

\begin{proposition}
The functions $\{\widehat{\Xi},\widehat{\Theta}_k\}$ satisfy the following:

\noindent(i) Upper bound: 
$$\widehat{\Xi}(\mathcal X|\mathcal X^\star )\geq \Xi(\mathcal{X}), \ \widehat{\Theta}_{k}(\mathcal{X},\mathcal{P}|\mathcal{P}^\star)\geq 
{\Theta}_{k}(\mathcal{X},\mathcal{P}).$$

\noindent(ii) Convexity: $\widehat{\Xi}(\mathcal X|\mathcal X^\star )$ is convex in $\mathcal{X}$, and $\widehat{\Theta}_{k}(\mathcal{X},\mathcal{P}|\mathcal{P}^\star)$ is jointly convex in $(\mathcal{X},\mathcal{P})$.

\noindent(iii) Local equivalence: 
\begin{align}
    \widehat{\Xi}(\mathcal X^\star|\mathcal X^\star )= \Xi(\mathcal{X}^\star), \ &  \nabla_{\mathcal{X}}\widehat{\Xi}(\mathcal X^\star|\mathcal X^\star )= \nabla_{\mathcal{X}}\Xi(\mathcal{X}^\star),
        \nonumber\\
    \widehat{\Theta}_{k}(\mathcal{X}^\star,\mathcal{P}^\star|\mathcal{P}^\star) &  = {\Theta}_{k}(\mathcal{X}^\star,\mathcal{P}^\star), 
    \nonumber\\
    \nabla_{(\mathcal{X},\mathcal{P})}\widehat{\Theta}_{k}(\mathcal{X}^\star,\mathcal{P}^\star|\mathcal{P}^\star) & = \nabla_{(\mathcal{X},\mathcal{P})} {\Theta}_{k}(\mathcal{X}^\star,\mathcal{P}^\star).
\end{align}
\end{proposition}
\begin{proof}
{Part (i) follows from
$$
\widehat{\Xi}(\mathcal X|\mathcal X^{[n]} )-\Xi(\mathcal X)
=
\frac{1}{\beta}\sum_{k=1}^K(x_{k}-x_{k}^{[n]})^2\geq 0,$$
$$
\widehat{\Theta}_{k}(\mathcal{X},\mathcal{P}|\mathcal{P}^\star)- 
{\Theta}_{k}(\mathcal{X},\mathcal{P})
\geq 0.
$$
Part (ii) follows from the positive-semidefinite Hessians of
$\widehat{\Xi}$ and $\widehat{\Theta}_k$. Part (iii) follows by evaluating
their function values and gradients at the expansion point.}
\end{proof}

By part (i) of \textbf{Proposition 1}, we obtain an upper bound by replacing the functions $\{\Xi,\Theta_k\}$ by $\{\widehat{\Xi},\widehat{\Theta}_k\}$ around a feasible point.
We tighten the bound iteratively by using each solution as the expansion point for the next surrogate.
In particular, assuming that the solution at the $n^{\mathrm{th}}$ iteration is given by $\{\mathcal{X}^{[n]},\mathcal{P}^{[n]}\}$, the following problem is considered at the $(n+1)^{\mathrm{th}}$ iteration:
\begin{subequations}
\begin{align}
\mathsf{P}2[n+1]:\mathop{\mathrm{min}}_{\substack{\mathcal{X},\mathcal{P}}}
\quad& {-\sum_{k=1}^K\eta_k\texttt{QoM}_k x_k}
+
\widehat{\Xi}(\mathcal X|\mathcal X^{[n]})  \label{P2n+1}
,\\
 \! \!\! \! \textrm{s.t.} ~~~~ & 
\widehat{\Theta}_{k}(\mathcal{X},\mathcal{P}|\mathcal{P}^{[n]})\leq 0, \ \forall k, \\
&  \textsf{constraints } (\ref{P2c}),(\ref{P2d}).
\end{align}
\end{subequations}

By part (ii) of \textbf{Proposition 1}, the problem $\mathrm{P}2[n+1]$ is convex and can be solved by off-the-shelf software packages (e.g., Mosek) for convex programming.
Denote its optimal solution as
$\{\mathcal{X}^*,\mathcal{P}^*\}$. 
Then we set
$\mathcal{X}^{[n+1]}=\mathcal{X}^*$ and $\{\mathcal{P}^{[n+1]}=\mathcal{P}^*\}$, and repeat the process by solving $\mathrm{P}2[n+2]$.
By part (iii) of \textbf{Proposition 1} and
\cite[Theorem 1]{sun2016majorization}, every limit point of the generated
sequence is a stationary point of $\mathrm{P}2$ when the initial point is
feasible. After convergence, we apply thresholding to the relaxed selection variables to obtain binary UAV-selection decisions.

PSO starts from a feasible point and iteratively solves $\mathrm{P}2[n+1]$ until both the objective value and optimization variables converge.
Each subproblem involves $2K$ decision variables, and the computational cost of the interior-point solver grows polynomially with $K$.
The repeated optimization required by PSO may therefore limit online responsiveness.

To mitigate this computational cost, we develop an L2M solver following the learning-to-optimize paradigm \cite{shlezinger2023model,Liu2024survey}.
Solutions obtained by PSO are used as supervision for training.
The input features include task-weighted QoM values, memory payload sizes, desired-link and interference gains, noise power, and power budgets, while the labels consist of the corresponding joint UAV-selection and power-allocation decisions.
A three-layer fully connected network with 100, 72, and 20 hidden units is trained using focal loss and AdamW.
During online inference, the server feeds the current memory-utility and channel-state features into the trained network.
A constraint-aware decoder then masks infeasible UAV selections and determines the corresponding transmit-power allocation.
The overall training and inference workflow is illustrated in Fig.~\ref{fig:fig3}.

To accommodate a varying number of UAVs, L2M adopts a fixed $K_{\max}$-slot representation.
Unused slots are zero-padded and their output logits are masked, enabling the same network architecture to support any $K\leq K_{\max}$.
Thus, L2M learns to reproduce PSO decisions under the same MemCen objective while substantially reducing online optimization overhead.

\begin{figure}[!t]
    \centering
    \includegraphics[width=0.48\textwidth]{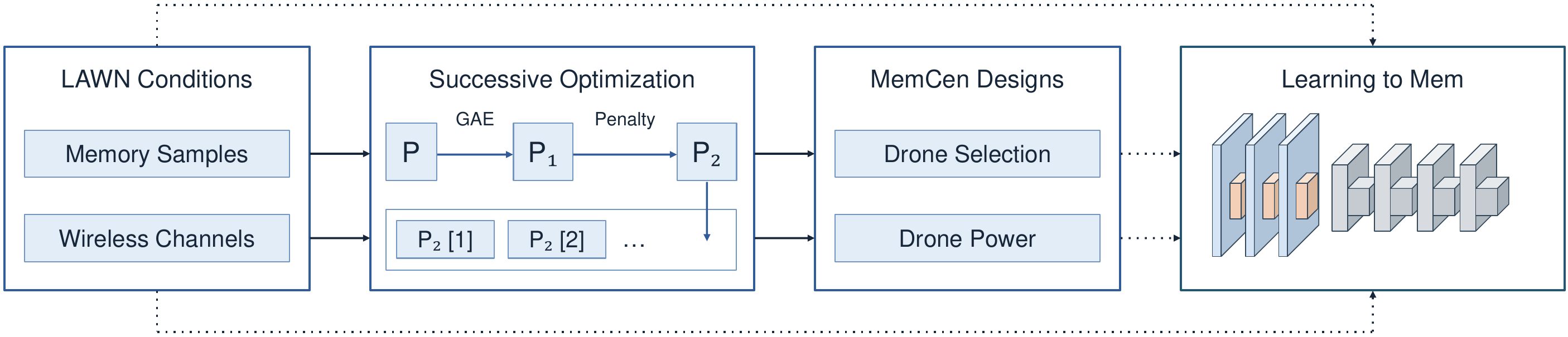}
    \caption{L2M for real-time decision making.}
    \label{fig:fig3}
    \vspace{-0.15in}
\end{figure}

\subsection{QoM-Aware Capped Water Filling}

To gain analytical insight into MemCen, we consider the asymptotic relaxation
of $\mathsf{P}2$ as $\beta\to\infty$. In this limit, the penalty term
$\Xi(\mathcal X)$ vanishes.
For any fixed $\mathbf p$, an optimal relaxed selection can be chosen as
\begin{align}
x_k^\star(\mathbf p)=\min\!\left\{1,
\frac{1}{\lambda_k}\log_2\!\left(1+\frac{H_kp_k}{
\sum_{j\ne k}I_{k,j}p_j+\sigma^2}\right)\right\}.
\label{eq:continuous-selection}
\end{align}
Substituting \eqref{eq:continuous-selection} into the objective removes the
explicit selection variables $\mathcal X$. We define
$\widehat T=T-T_{\mathrm{Pilot}}$ as the remaining upload time and
$\widehat S_k=S_k-\widetilde S_k$ as the remaining payload of UAV $k$. Under
the noise-limited condition $\sum_{j\ne k}I_{k,j}p_j\ll\sigma^2$,
$\mathsf{P}2$ reduces to the following power-allocation problem:
\begin{align}
\mathsf{P}3: \mathop{\mathrm{max}}_{\mathcal P} \ U(\mathbf p),
\quad \mathrm{s.t.} \ \eqref{P2c}.
\end{align}
where the utility as a function of the power allocation is
\begin{align}
U(\mathbf p)
&=\sum_{k=1}^{K}\eta_k\texttt{QoM}_k\min\!\left\{1,\frac{\widehat TB}{\widehat S_k}
\log_2\!\left(1+\frac{H_kp_k}{\sigma^2}\right)\right\}.
\label{eq:continuous-qom}
\end{align}
The following proposition characterizes the global optimum of $\mathsf{P}3$
and exposes its QoM-aware power-allocation structure.

\begin{proposition}\label{prop:qom-water-filling}
An optimal power allocation $\mathbf p^\star$ to $\mathsf{P}3$ is
\begin{subequations}\label{eq:qom-water-filling}
\begin{align}
p_k^\star(\nu)&=\min\!\left\{\bar p_k,
\left[\frac{\nu\eta_k\texttt{QoM}_k\widehat TB}{\widehat S_k}-\frac{\sigma^2}{H_k}\right]^+\right\},
\label{eq:qom-water-filling-allocation}\\
\bar p_k&=\min\!\left\{P_{\max},p_k^{\mathrm{sat}}\right\},\\
p_k^{\mathrm{sat}}&=\frac{\sigma^2}{H_k}
\left(2^{\widehat S_k/(\widehat TB)}-1\right),
\label{eq:qom-water-filling-cap}
\end{align}
\end{subequations}
where $p_k^{\mathrm{sat}}$ is the saturation power. If $\sum_{k:\eta_k\texttt{QoM}_k>0}\bar p_k\geq P_{\mathrm{sum}}$, $\nu$ is chosen
such that $\sum_kp_k^\star(\nu)=P_{\mathrm{sum}}$. Otherwise,
$p_k^\star=\bar p_k$ for $\eta_k\texttt{QoM}_k>0$ and $p_k^\star=0$ for $\eta_k\texttt{QoM}_k=0$.
\end{proposition}

\begin{proof}
The proof is based on the Karush--Kuhn--Tucker (KKT) conditions. Define
$a_k=\eta_k\texttt{QoM}_k$ and
$c_k=\widehat TB/\widehat S_k$. The $k$-th utility term is constant for
$p_k\geq p_k^{\mathrm{sat}}$. Hence, an optimal solution can be chosen with
$p_k\leq\bar p_k$. Over $0\leq p_k\leq\bar p_k$, each utility term is
nondecreasing and concave, so the KKT conditions are necessary and sufficient.
Let $\lambda\geq0$ be the multiplier of the sum-power constraint. For an
active and unsaturated UAV, stationarity gives
$\frac{a_kc_kH_k}{\ln 2\,(\sigma^2+H_kp_k)}=\lambda$,
which yields
$p_k=\frac{a_kc_k}{\lambda\ln 2}-\frac{\sigma^2}{H_k}$.
Applying complementary slackness to the lower and upper power bounds gives
\eqref{eq:qom-water-filling-allocation} with
$\nu=1/(\lambda\ln 2)$. Solving
$c_k\log_2(1+H_kp_k/\sigma^2)=1$ gives
$p_k^{\mathrm{sat}}$ in \eqref{eq:qom-water-filling-cap}. If the caps of the
positive-utility UAVs sum to at least $P_{\mathrm{sum}}$, $\nu$ is chosen to
meet the sum-power constraint with equality. Otherwise, each positive-utility
UAV receives $\bar p_k$, and the unused budget cannot increase the objective. 
This proves the stated solution.
\end{proof}

Proposition~\ref{prop:qom-water-filling} extends classical water filling with
a task-utility weight and an upload-completion cap. When the aggregate cap of
the positive-utility UAVs meets or exceeds the power budget, the water level $\nu$ can
be found by bisection so that the allocation exhausts $P_{\mathrm{sum}}$.
Otherwise, every positive-utility UAV receives its capped power, and any
remaining power provides no additional memory utility. Within the unsaturated
region, the allocated power is nondecreasing in $\eta_k\texttt{QoM}_k$ and is
favored by a smaller remaining payload $\widehat S_k$. A UAV with
$\texttt{QoM}_k=0$ receives no power, while the cap prevents additional
allocation after its upload is complete. Thus, QoM changes both UAV activation
and power-allocation priority beyond channel-only water filling.

\section{Simulation Results}\label{section5}

We implement the proposed methods in Python using CARLA \cite{carla}.
We develop OpenMAMS for the design and evaluation of LAQA systems.
OpenMAMS is an open-source multi-drone simulation platform\footnote{\url{https://github.com/SIAT-INVS/OpenMAMS}} that supports: (1) multi-agent multi-modal dataset generation, including timestamps, high-fidelity images, 6D poses, and lidar point cloud data; (2) memory generation for image data, including VLM captioning, text embedding, and vector database building;
(3) question answering framework, including LLM inference, chain-of-thought reasoning, and interactive interface;
(4) optimization algorithms, including PSO/L2M optimizers and eight other baselines.
The OpenMAMS platform is deployed on a Linux workstation with an NVIDIA RTX 5090 GPU.
Unless otherwise specified, we use \texttt{Qwen3-VL-8B} for image captioning, \texttt{Mxbai-Embed-Large-v1} for text embedding, \texttt{Milvus} for top-5 memory retrieval, and \texttt{Qwen3-8B} for question answering. We report the sum GAE score $\sum_{k=1}^{K}\texttt{Score}_k x_k$ and the sum QoM $\sum_{k=1}^{K}\texttt{QoM}_k x_k$, and compare the following methods
under the same setting within each run:
\begin{itemize}
   \item \textbf{ComCen} \cite{ye2025integrated}: maximizes the sum rate.\footnote{We adopt the objective function and power constraint in \cite{ye2025integrated} and omit its trajectory and sensing constraints.}
   \item \textbf{SenCen} \cite{liu2024coverage}: maximizes sensing coverage.
   \item \textbf{FairCen} \cite{zheng2016wireless}: maximizes the minimum rate.
   \item \textbf{Greedy} \cite{li2024survey}: selects UAVs by GAE score without joint channel optimization.
   \item \textbf{Remember} \cite{anwar2025remembr}: retrieves only from $\mathcal M_0$.
   \item \textbf{SemCom} \cite{liu2025intelligent}: ranks memories by embedding similarity.
   \item \textbf{MemCen+PSO}: solves MemCen by penalty successive optimization.
   \item \textbf{MemCen+L2M}: predicts MemCen decisions with the learned L2M surrogate.
   \item \textbf{MemCen+Relax-Round}: relaxes and then rounds the selection variables.
   \item \textbf{MemCen+DQN}: uses a deep Q-learning neural network with a feasibility mask.
\end{itemize}

\begin{figure*}[!t]
    \centering
    \begin{minipage}[t]{0.56\textwidth}
        \centering
        \begin{minipage}[c][2.25in][c]{\linewidth}
            \centering
            \includegraphics[width=\linewidth]{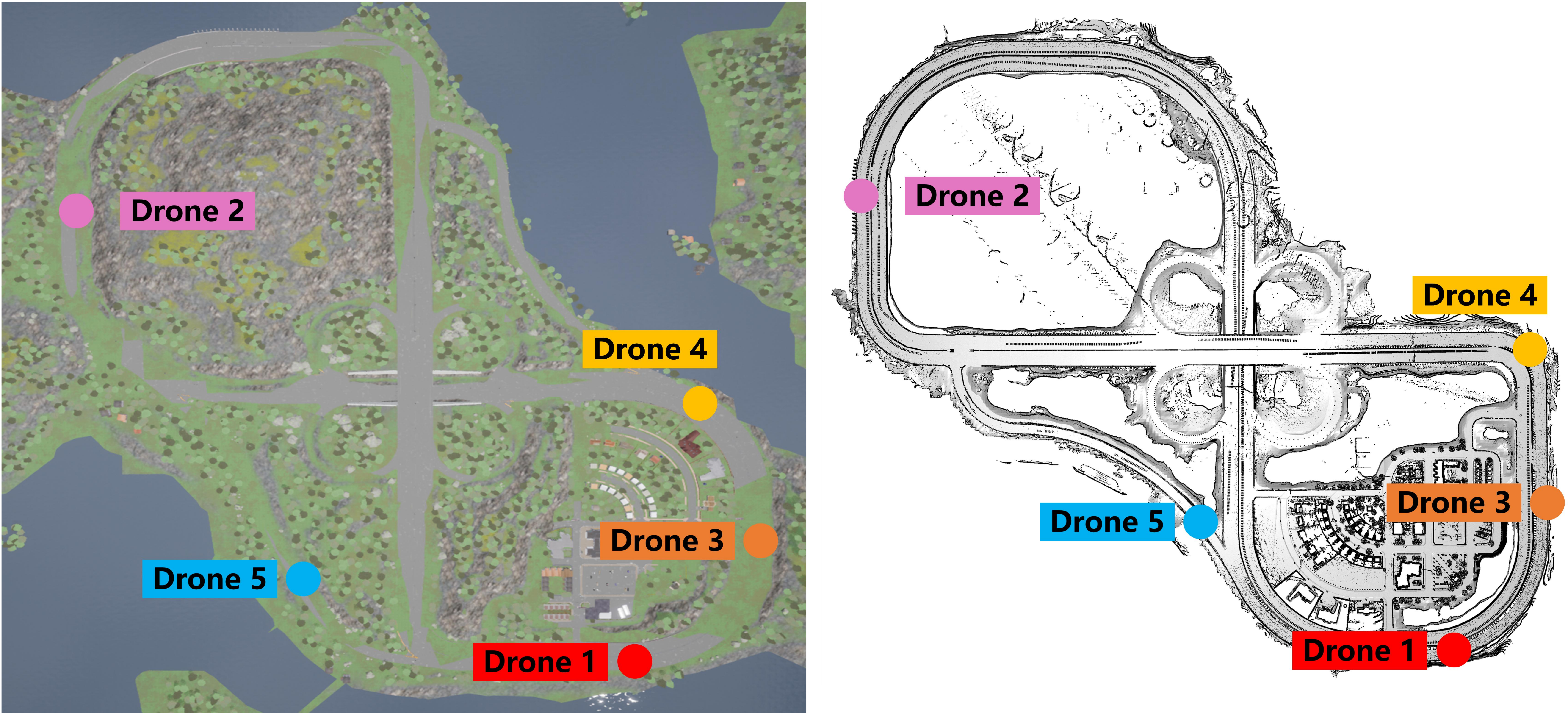}
        \end{minipage}
        \vspace{-0.15in}
        \caption{Simulation settings for the $5$-UAV Town04 scenario.}
        \label{fig:fig4}
    \end{minipage}\hfill
    \begin{minipage}[t]{0.40\textwidth}
        \centering
        \begin{minipage}[c][2.25in][c]{\linewidth}
            \centering
            \begin{subfigure}{\linewidth}
                \centering
                \includegraphics[width=\linewidth]{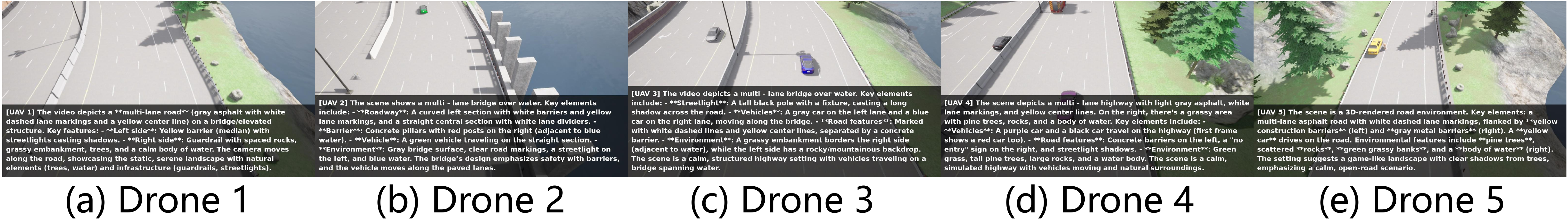}
                \caption{Visualization of five UAV frames and captions.}
            \end{subfigure}
            \par\vspace{0.2in}
            \begin{subfigure}{\linewidth}
                \centering
                \includegraphics[width=\linewidth]{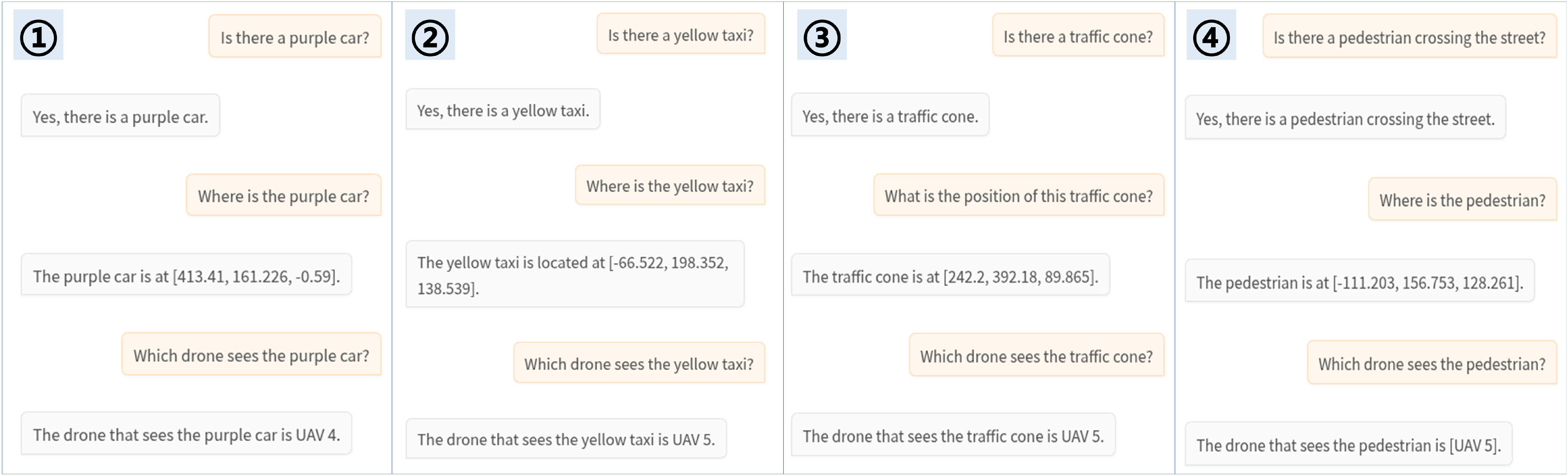}
                \caption{Visualization of LAQA.}
            \end{subfigure}
        \end{minipage}
        \vspace{-0.15in}
        \caption{Captioning and QA in the $5$-UAV scenario.}
        \label{fig:fig5}
    \end{minipage}
\end{figure*}

\begin{table*}[t] 
  \centering
  \caption{Evaluation of GAE. GAE scores are mean unanswered-question counts for five-question exams. \colorbox{purple2}{Purple} denotes the \textbf{best} memory selected by the method. \textbf{Bold} denotes the \textbf{highest} value.}
  {
  \resizebox{0.9\textwidth}{!}{%
  \begin{tabular}{ccccccc}
    \toprule
    \multirow{2}{*}{UAV ID} 
    & SemCom 
    & \multicolumn{2}{c}{\texttt{Qwen3-8B} ($6.43$\,s per QA)}
    & \multicolumn{2}{c}{\texttt{Qwen3-14B} ($12.17$\,s per QA)} 
    & \texttt{Qwen3-8B} \\
    \cmidrule(lr){2-2} \cmidrule(lr){3-4} \cmidrule(lr){5-6} \cmidrule(lr){7-7} 
    & \cite{liu2025intelligent} & GAE Score & GAE QA Accuracy 
    & GAE Score & GAE QA Accuracy & Downstream QA Accuracy \\
    \midrule
    UAV 01 & 0.8913 & 1.05 & \maxval{\textbf{79\%}}  & 0.60 & 88\%  & 39\% \\
    UAV 02 & \cellcolor{purple2}0.8664 & 2.75 & 45\%   & 1.65 & 67\%  & 48\% \\
    UAV 03 & 0.8746 & 3.25 & 35\%  & 1.20 & 76\%  & 59\% \\
    UAV 04 & 0.9009 & \cellcolor{purple2} \maxval{\textbf{3.65}} & \cellcolor{purple2}27\%  & \cellcolor{purple2} \maxval{\textbf{2.25}} & \cellcolor{purple2}55\%  & \cellcolor{purple2}\maxval{\textbf{70\%}} \\
    UAV 05 & \maxval{\textbf{0.9394}} & 1.40 & 72\%  & 0.20 & \maxval{\textbf{96\%}}  & 38\% \\
    \bottomrule
  \end{tabular}
  }
  }
  \label{tab:tab1}
\end{table*}

\subsection{Town04 Evaluation}

We consider the Town04 map in CARLA. 
We randomly generate $10$ objects (i.e., [green car, blue car, white car, purple car, black car, fire truck, motorcycle, bus, taxi, traffic cone]) on the road. The questions cover $3$ types, with $30$ questions in each realization. 
For a fire truck, the questions are:
\begin{itemize}
    \item[1)] Is there a fire truck? The answer should be YES/NO.
    \item[2)] Where is the fire truck? The answer should be an explicit [x,y] coordinate that is within $50$\,m from the ground-truth object location. 
    \item[3)] Which UAV sees the fire truck? The answer should be the UAV ID $k$.
\end{itemize}
The inspection UAVs are randomly distributed using the default CARLA map spawn points. 
Each UAV conducts inspections within a $500$-meter range starting from its spawn point. 
Each frame has a measured data volume of $V_{i,k}=1600$\,kbits (i.e., $200$\,kB). 
The frame rate is set to $35$\,FPS and $1050$ frames are recorded for each UAV, with a total of $1050K$ frames generated in a single run. 
The GAE exam uses $L_k\in\{5,10\}$ generated questions. The pilot sampling ratio is
$\rho_k\approx1\%$, implemented by sampling $10$ images per UAV.

For the air-to-ground communication link, the system bandwidth is set to $B=10\,\mathrm{MHz}$ and the noise power to $\sigma^2=-70\,\mathrm{dBm}$ \cite{wang2020angle}.
The wireless collection budget is $T=600\,\mathrm{s}$, i.e., 10 minutes.
The total transmit-power budget is set to $P_{\mathrm{sum}}=300\,\mathrm{mW}$, while the maximum per-UAV transmit power is $P_{\mathrm{max}}=100\,\mathrm{mW}$.
The ground server is located at the map center $(0,0)$ at a height of $20\,\mathrm{m}$, and the UAV-to-server distance is sampled as $d_k\sim\mathcal{U}[50,250]\,\mathrm{m}$.
We assume that UAV inspection and wireless transmission are performed in separate phases.
For dense-building scenarios, the air-to-ground channel follows a Rayleigh fading model,
$\mathbf{h}_k = \sqrt{h_0 \omega_k d_k^{-\alpha}}\mathbf{h}^{\text{CN}}_{k}$, where $h_0=-30\,\mathrm{dB}$ denotes the reference path loss at $1\,\mathrm{m}$, $\omega_k=-20\,\mathrm{dB}$ denotes the shadowing factor of UAV $k$, and $\alpha=3$ is the path-loss exponent.
The small-scale fading component satisfies
$\mathbf{h}_k^{\mathrm{CN}}\sim\mathcal{CN}(\mathbf{0},\mathbf{I}_N)$, with $N=256$ antennas at the ground server.
MRC is adopted unless otherwise stated.
Unless otherwise specified, quantitative results are averaged over 20 independent simulation runs, with independently generated object locations and wireless channels in each run.

\begin{figure*}[t]
    \centering
    \includegraphics[width=1\textwidth]{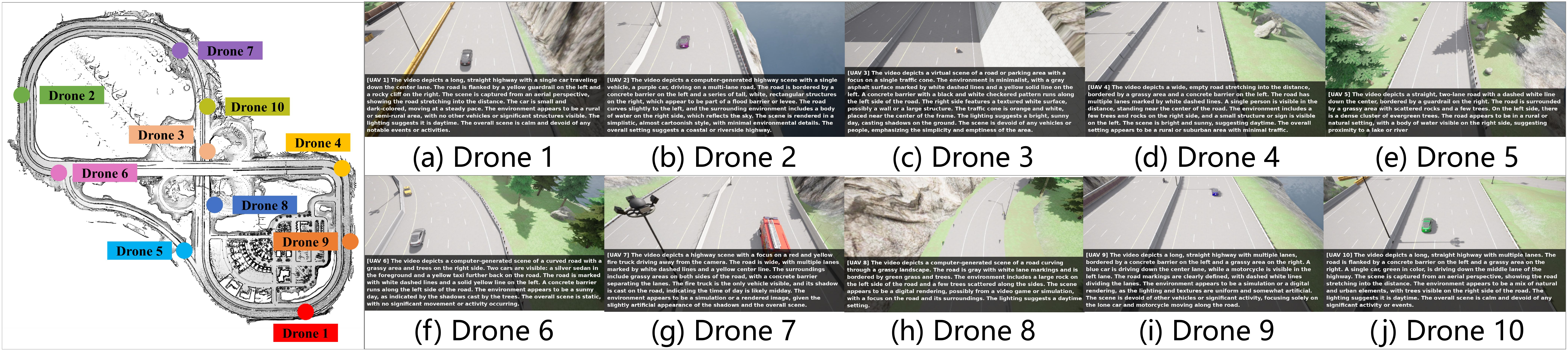}
    \caption{Visualization of ten UAV configurations, image frames, and the associated captions.}
    \label{fig:fig6}
\end{figure*}

\subsection{Validation of the GAE}\label{section5-1}

We first test whether GAE predicts downstream memory utility.
We consider the case of $K=5$ and the UAV spawn points are illustrated in Fig.~\ref{fig:fig4}. 
To simulate the heterogeneous memories, we place $[0,1,2,3,4]$ abnormal objects/events in inspection regions of UAVs $[1,2,3,4,5]$, and the VLM captioning results are illustrated in Fig. \ref{fig:fig5}(a). 
The captions show accurate scene understanding with the VLM \texttt{Qwen3-VL-8B}.

We set the historical memory to $\mathcal{M}_0=\mathcal{M}_5$ and evaluate the candidate memories of UAVs 1--5 over 20 random runs.
In each run, $L_k=5$ questions are generated from 10 randomly sampled images for each UAV.
Table~\ref{tab:tab1} reports the average GAE scores and QA accuracies obtained with \texttt{Qwen3-8B} and \texttt{Qwen3-14B}.
Both SemCom and GAE correctly identify the high similarity between the sampled memory of UAV~5 and the historical memory $\mathcal{M}_0$.
However, SemCom provides similar scores for UAVs 1--4 and therefore has limited ability to distinguish their relative memory values.
In contrast, 
under \texttt{Qwen3-8B}, the GAE scores progressively increase from UAV~1 to UAV~4, consistent with the increasing semantic richness of their observations.
We further observe that \texttt{Qwen3-14B} achieves consistently higher QA accuracy than \texttt{Qwen3-8B}. Nevertheless, both models clearly differentiate heterogeneous memory qualities.

We then evaluate downstream QA using each memory union
$\mathcal{M}_k\cup\mathcal{M}_0$ with \texttt{Qwen3-8B}.
GAE selects UAV~4 under both answering models, yielding 70\% downstream
QA accuracy. SemCom selects UAV~2, yielding 48\%. The comparison supports
using GAE to select useful additional memory in this scenario.
The QA results for memory $\mathcal{M}_4\cup\mathcal{M}_0$ are shown in Fig.~\ref{fig:fig5}(b). 
All answers about the presence, position, and reporting UAVs of the purple car, yellow taxi, traffic cone, and pedestrian are correct. 

\subsection{Evaluation of PSO and L2M}

\begin{figure*}[!t]
    \centering
    \begin{subfigure}{0.196\linewidth}
        \centering
        \includegraphics[width=\linewidth]{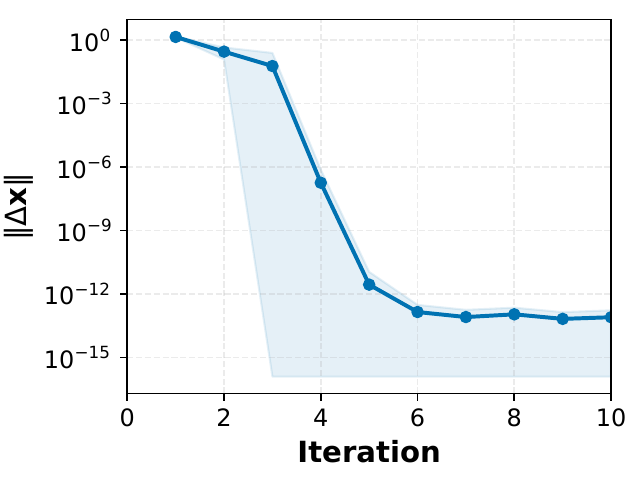}
        \caption{$\|\Delta\mathbf{x}\|$ versus $n$.}
    \end{subfigure}%
    \begin{subfigure}{0.196\linewidth}
        \centering
        \includegraphics[width=\linewidth]{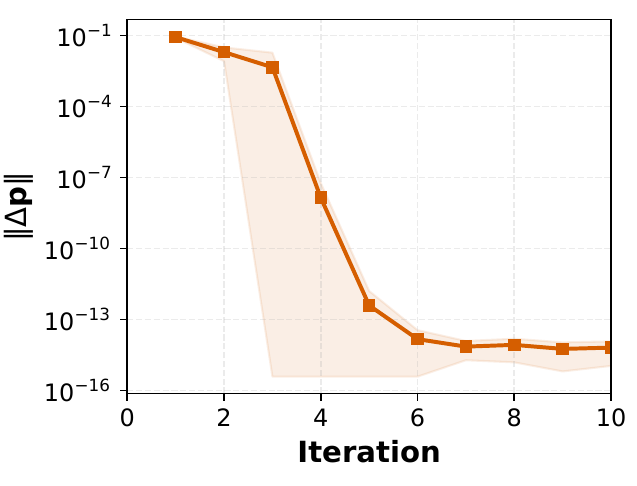}
        \caption{$\|\Delta\mathbf{p}\|$ versus $n$.}
    \end{subfigure}%
    \begin{subfigure}{0.196\linewidth}
        \centering
        \includegraphics[width=\linewidth]{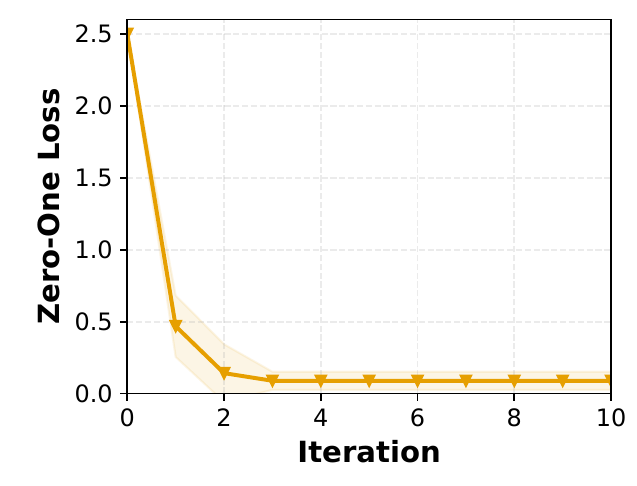}
        \caption{Zero-one loss versus $n$.}
    \end{subfigure}%
    \begin{subfigure}{0.196\linewidth}
        \centering
        \includegraphics[width=\linewidth]{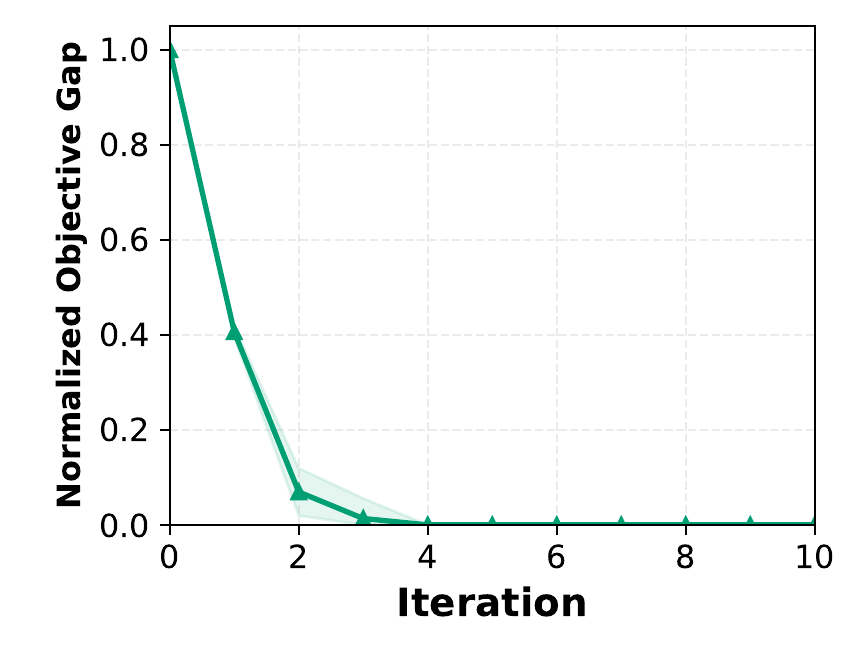}
        \caption{Objective gap.}
    \end{subfigure}%
    \begin{subfigure}{0.196\linewidth}
        \centering
        \includegraphics[width=\linewidth]{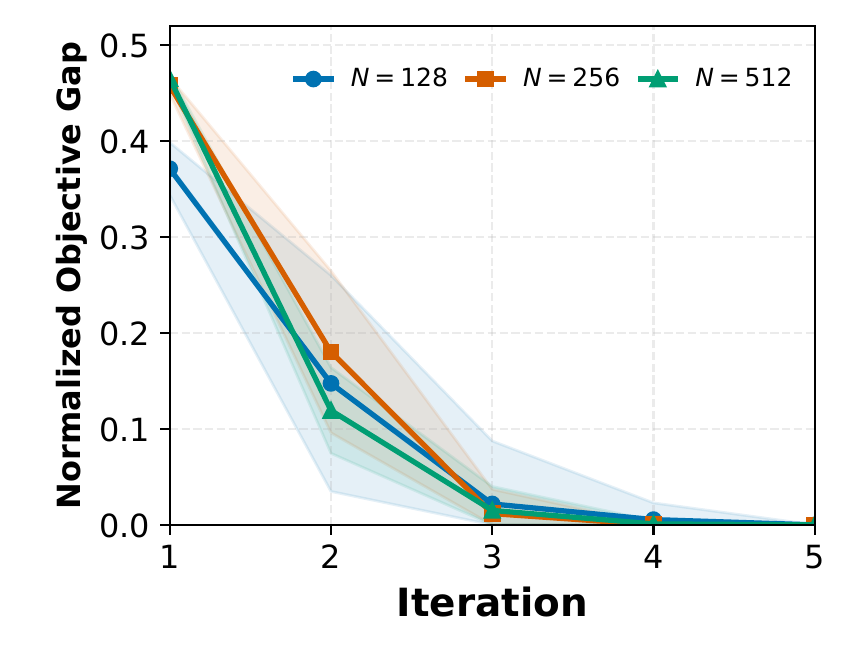}
        \caption{Different $N$.}
    \end{subfigure}%
    \caption{Convergence analysis of PSO.}
    \label{fig:fig7}
    \vspace{-0.1in}
\end{figure*}

\begin{table*}[!t]
    \centering
    \caption{{Comparison of solvers under the same MemCen formulation.}}
    \label{tab:tab2}
    \setlength{\tabcolsep}{3.2pt}
    \renewcommand{\arraystretch}{1.12}
    \footnotesize
    {
    \begin{tabular*}{0.9\textwidth}{@{\extracolsep{\fill}}lcccc@{}}
        \toprule
        Method & {$\uparrow$Sum GAE} & UAVs &
        {$\uparrow$Sum rate (Mbps)} & {$\downarrow$Time (s)} \\
        \midrule
        \textbf{MemCen+PSO} & \maxval{\textbf{$24.62\!\pm\!9.00$}} & \maxval{$3.50\!\pm\!1.41$} &
        $45.03\!\pm\!16.57$ & $4.60\!\pm\!0.40$ \\
        \textbf{MemCen+L2M} & $23.38\!\pm\!12.99$ & $3.10\!\pm\!1.69$ &
        $42.44\!\pm\!21.91$ & \textbf{$0.04\!\pm\!0.03$} \\
        MemCen+Relax-Round & $20.75\!\pm\!7.12$ & $3.20\!\pm\!1.20$ &
        $35.26\!\pm\!14.19$ & $9.16\!\pm\!0.43$ \\
        MemCen+DQN & $22.20\!\pm\!8.83$ & $2.93\!\pm\!1.14$ &
        \maxval{$47.81\!\pm\!18.67$} & \maxval{$0.017\!\pm\!0.0003$} \\
        \bottomrule
    \end{tabular*}}
\end{table*}

We examine the convergence of PSO with $P_{\mathrm{sum}}=200$\,mW and
$K=10$ as illustrated in Fig.~\ref{fig:fig6}. We define the consecutive-iteration changes as
$\Delta\mathbf{x}=\mathbf{x}^{[n]}-\mathbf{x}^{[n-1]}$ and
$\Delta\mathbf{p}=\mathbf{p}^{[n]}-\mathbf{p}^{[n-1]}$. As shown in
Figs.~\ref{fig:fig7}(a) and \ref{fig:fig7}(b), both
$\|\Delta\mathbf{x}\|$ and $\|\Delta\mathbf{p}\|$ fall below $10^{-8}$
after five iterations. Fig.~\ref{fig:fig7}(c) reports the total zero-one
residual $\phi(\mathbf{x})=\sum_{k=1}^{K}x_k(1-x_k)$.
The residual falls below $0.1$ within five iterations.
Fig.~\ref{fig:fig7}(d) shows that the normalized
objective gap of $\mathsf{P}2$ approaches zero within five iterations.
We examine the sensitivity of binary feasibility to the penalty parameter over ten channel realizations.
For $\beta\in\{0.01,0.05,0.1,0.5,1,2,5\}$, the final mean total zero-one residual ranges from $0.073$ to $0.080$ when $\beta\leq0.5$, while it exceeds $0.1$ for $\beta\geq1$.
Thus, all tested values with $\beta\leq0.5$ satisfy the near-binary criterion of $0.1$ for $K=10$.
We select $\beta=0.5$, which yields a residual of $0.080$ and represents the upper end of the range that maintains near-binary feasibility.

We further investigate whether the number of server antennas affects PSO convergence.
We vary $N\in\{128,256,512\}$ under a total power budget of $P_{\mathrm{sum}}=200$~mW and average the results over ten random channel realizations.
All runs are initialized with $x_k^{[0]}=0.5$ and $p_k^{[0]}=P_{\mathrm{sum}}/K$.
As shown in Fig.~\ref{fig:fig7}(e), the mean normalized objective gap approaches zero within five iterations for all three antenna configurations.
From iterations 4 to 5, the relative objective changes are only $0.53\%$, $0.009\%$, and $0.10\%$ for $N=128$, $256$, and $512$, respectively.
Although increasing $N$ improves the attained objective value, it does not materially slow PSO convergence, since the dimensions of the optimization variables $\mathbf{x}$ and $\mathbf{p}$ remain unchanged.

We then benchmark the solver performance under the same MemCen objective and power budget.
Table~\ref{tab:tab2} reports the results over 20 random IRC channel realizations.
PSO achieves the highest average sum GAE score.
L2M retains 95.0\% of the PSO performance while reducing the average execution time from 4.60~s to 0.04~s, corresponding to a $115\times$ speedup.
In comparison, Relax-Round incurs a 15.7\% performance loss relative to PSO.
DQN also achieves a lower GAE score than L2M, as it learns solely from reward feedback without access to expert demonstrations generated by PSO.
These results demonstrate that PSO provides high-quality solutions to the proposed MemCen formulation, while L2M effectively amortizes the optimization cost for low-latency inference.

We further test L2M under UAV-density shifts. The fixed-capacity
model is trained only at $K=10$ and tested without fine-tuning at
$K\in\{5,10,15,20\}$. Across three model seeds,
it retains $100.00\%$, $99.80\%$, $98.93\%$, and $97.46\%$ of the
branch-and-bound optimal QoM, respectively. The selected-UAV counts closely
match the optimum, and all decoded decisions satisfy the communications constraints.

\subsection{Comparison with Existing Benchmarks}

\begin{figure*}[!t]
    \centering
    \begin{subfigure}{0.235\linewidth}
        \centering
        \includegraphics[width=\linewidth]{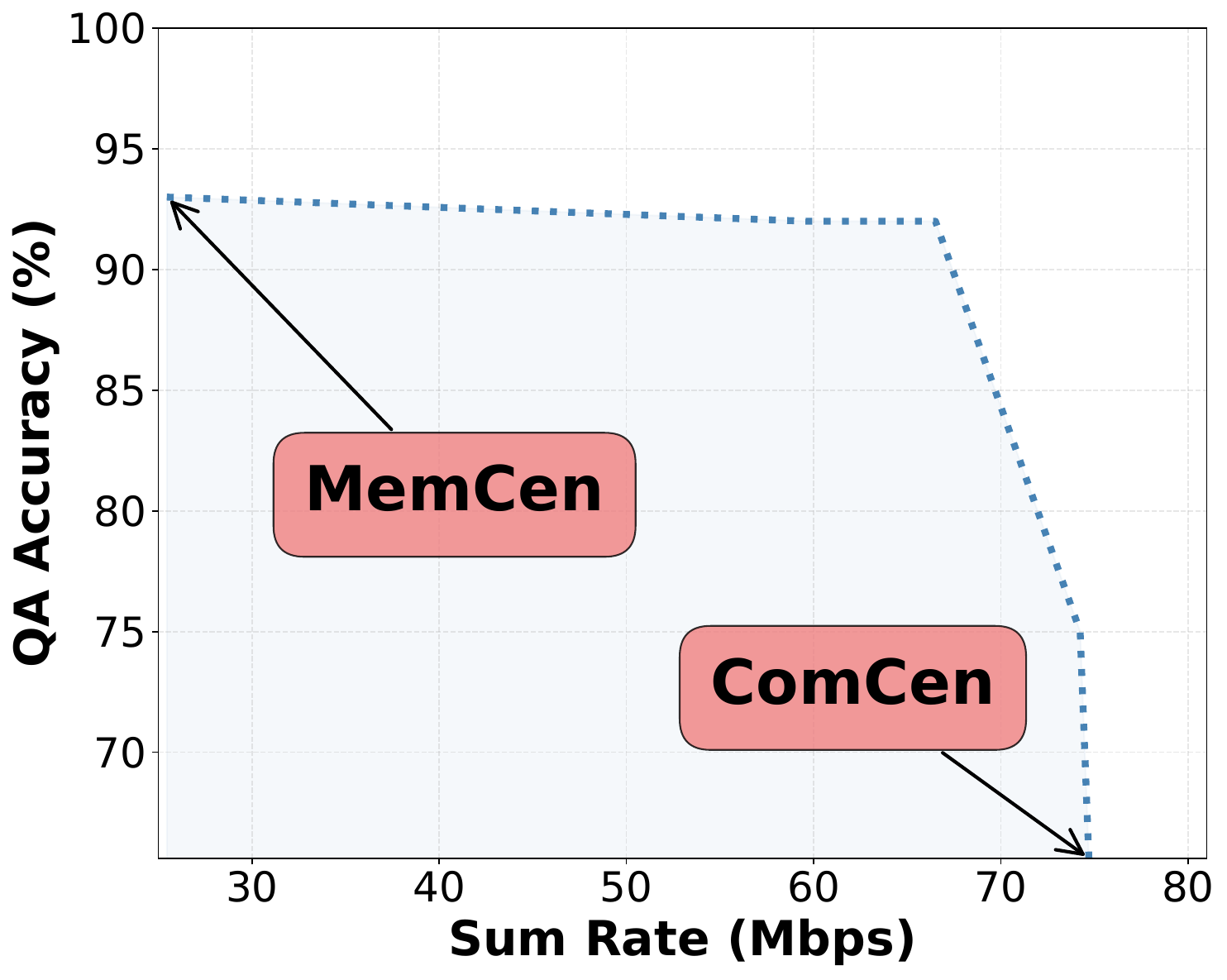}
        \caption{Accuracy--rate region.}
    \end{subfigure}\hfill
    \begin{subfigure}{0.235\linewidth}
        \centering
        \includegraphics[width=\linewidth]{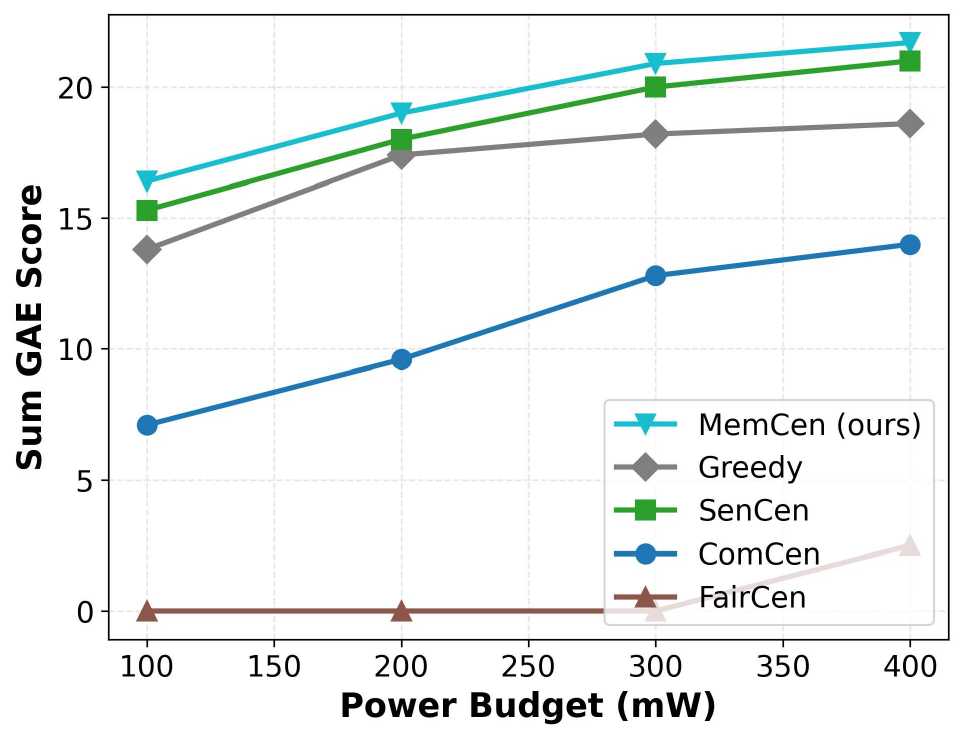}
        \caption{Sum GAE score.}
    \end{subfigure}\hfill
    \begin{subfigure}{0.235\linewidth}
        \centering
        \includegraphics[width=\linewidth]{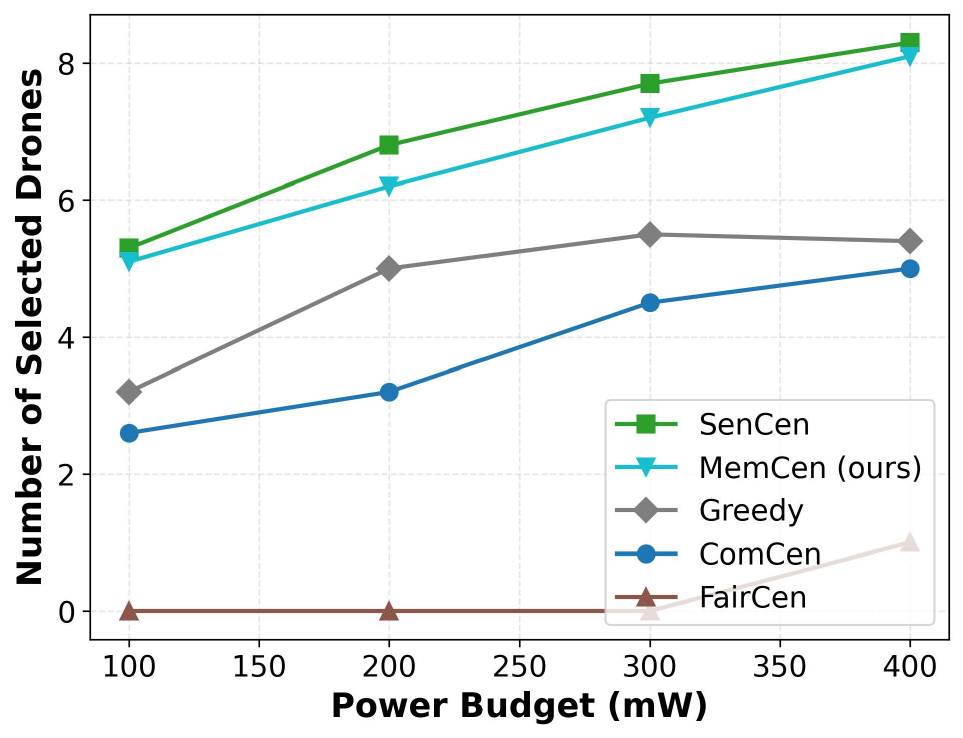}
        \caption{Selected UAVs.}
    \end{subfigure}\hfill
    \begin{subfigure}{0.235\linewidth}
        \centering
        \includegraphics[width=\linewidth]{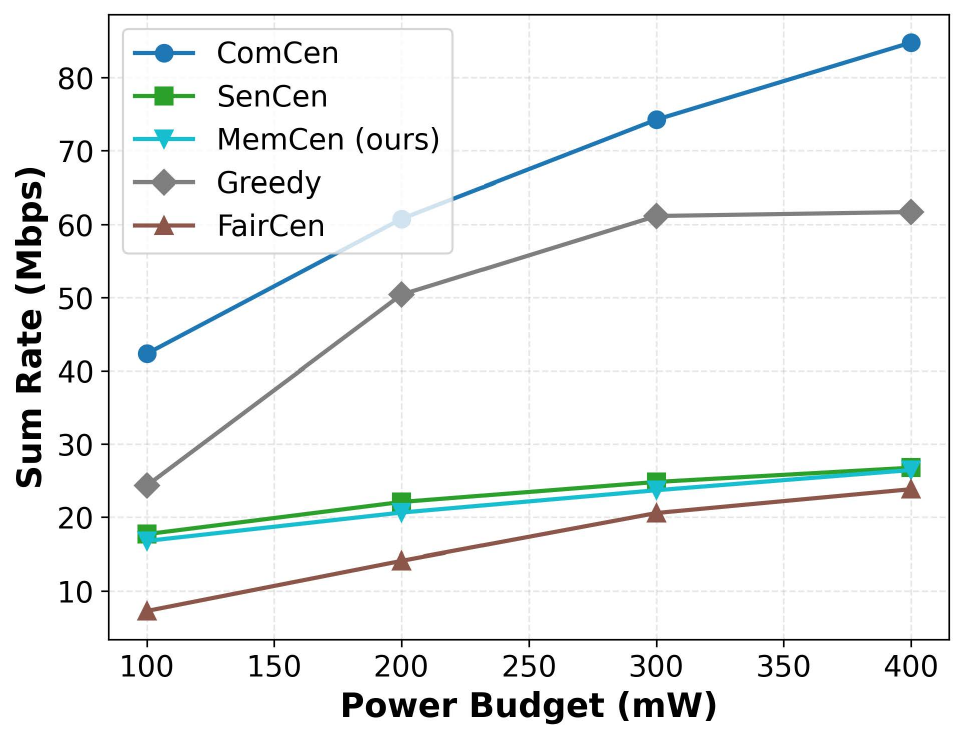}
        \caption{Sum rate.}
    \end{subfigure}
    \caption{Accuracy--rate tradeoff and power-budget sensitivity.}
    \label{fig:fig8}
    \vspace{-0.15in}
\end{figure*}

\begin{table}[!t] 
\small
\centering
\setlength{\tabcolsep}{1.5mm}
\renewcommand{\arraystretch}{1.35}
\caption{Quantitative results of QA tasks.}
\scalebox{0.8}{
\begin{tabular}{lc|cccc}
\toprule
\multicolumn{2}{c}{Metric} & QA Accuracy $\uparrow$ & Sum GAE Score $\uparrow$ & UAVs$\uparrow$ & Sum Rate$\uparrow$ \\
\midrule
ComCen & \cite{ye2025integrated} &65.6\% & 10.2 & 4.1 & \cellcolor{purple2}\maxval{74.70\,Mbps} \\
SenCen & \cite{liu2024coverage} & 88.6\% & 19.9 & \cellcolor{purple2} \maxval{8.1} & 25.48\,Mbps  \\
FairCen & \cite{zheng2016wireless} & 46.0\% & 2.5 & 1.0 & 21.46\,Mbps \\
Greedy & \cite{li2024survey}  & 89.8\% & 17.3 & 4.4 & 44.56\,Mbps \\
Remember & \cite{anwar2025remembr} & 41.0\% & -- & -- & --  \\
MemCen & (\textbf{Ours}) & \cellcolor{purple2}\maxval{\textbf{92.4\%}} & \cellcolor{purple2}\maxval{\textbf{20.7}} & \textbf{7.4} & 24.31\,Mbps  \\

\bottomrule
\end{tabular}
}
\label{tab:tab3}
\end{table}

\begin{table}[!t]
\centering
\caption{Evaluation of the $10$-UAV Case}
\label{tab:tab4}
\scalebox{0.72}{
\begin{tabular}{l c | c c c c}
    \toprule
    \multicolumn{2}{c}{Method}  & {$\uparrow$Sum GAE Score} & {$\uparrow$UAVs} & {$\uparrow$Sum Rate} & {$\uparrow$Min. Rate} \\
    \midrule
    ComCen & \cite{ye2025integrated}         & 12 & 4 & \cellcolor{purple2} \maxval{63.9\,Mbps} & 0\,Mbps \\
    SenCen & \cite{liu2024coverage}        & 16 & \cellcolor{purple2} \maxval{7} & 21.9 \,Mbps& 0\,Mbps \\
    FairCen & \cite{zheng2016wireless} & 0 & 0 & 17.2\,Mbps & \cellcolor{purple2} \maxval{1.72\,Mbps}  \\
    Greedy & \cite{li2024survey} & 16 & 3 & 33.5\,Mbps & 0\,Mbps\\
    MemCen & (Ours)          & \cellcolor{purple2} \maxval{\textbf{21 (+$31.25\%$)}}
                    & \textbf{6}
                    & \textbf{20.1}\,Mbps & 0\,Mbps \\
    \bottomrule
\end{tabular}
}
\end{table}

\begin{figure}[!t]
    \centering
    \includegraphics[width=0.49\textwidth]{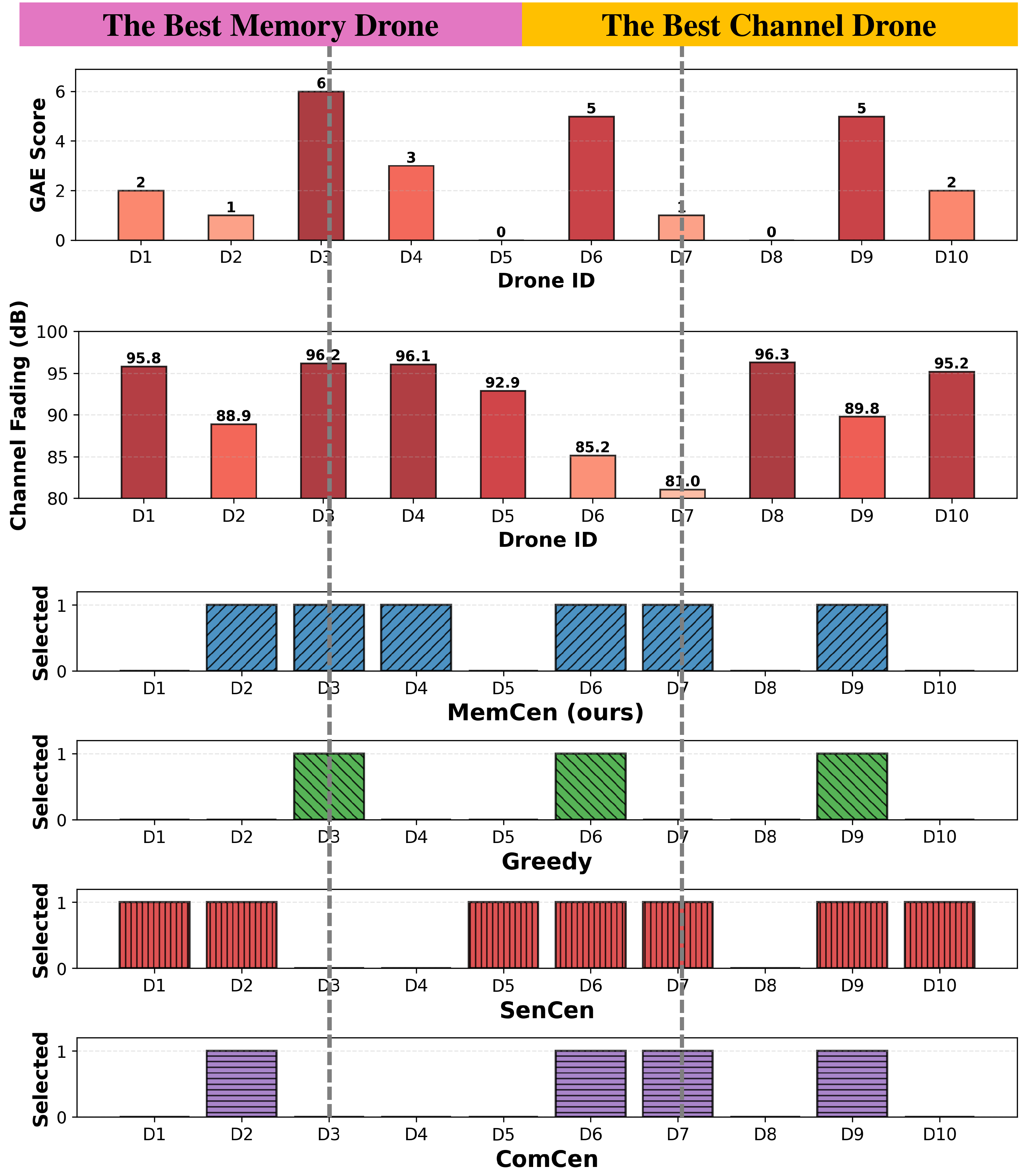}
    \caption{GAE scores, channel conditions, and UAV selections of different schemes.}
    \label{fig:fig9}
    \vspace{-0.2in}
\end{figure}

\begin{figure*}[!t]
	\centering
	\begin{subfigure}{0.22\linewidth}
		\centering
		\includegraphics[width=\linewidth]{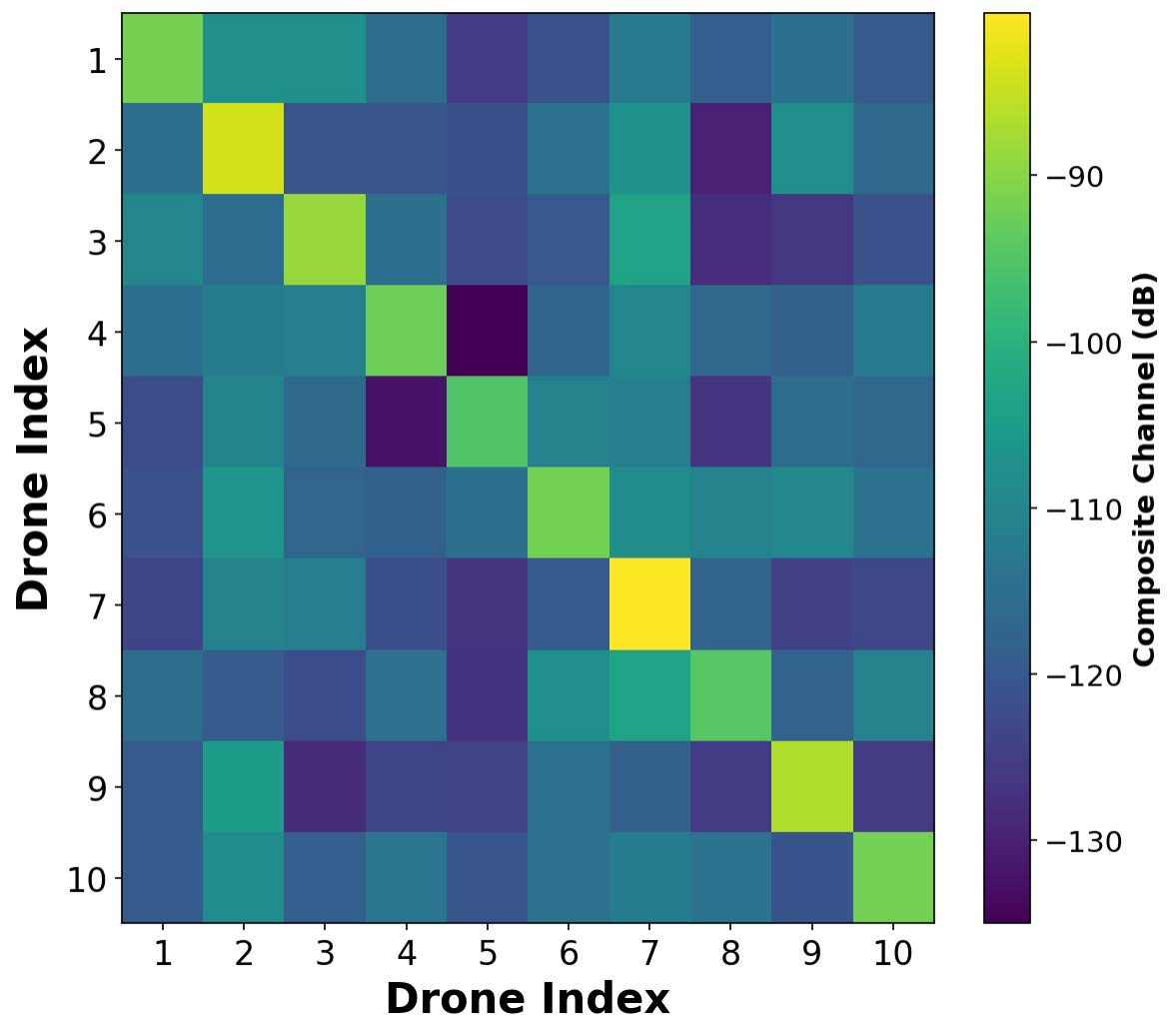}
		\caption{MRC channel map.}
	\end{subfigure}
 	\begin{subfigure}{0.22\linewidth}
		\centering
		\includegraphics[width=\linewidth]{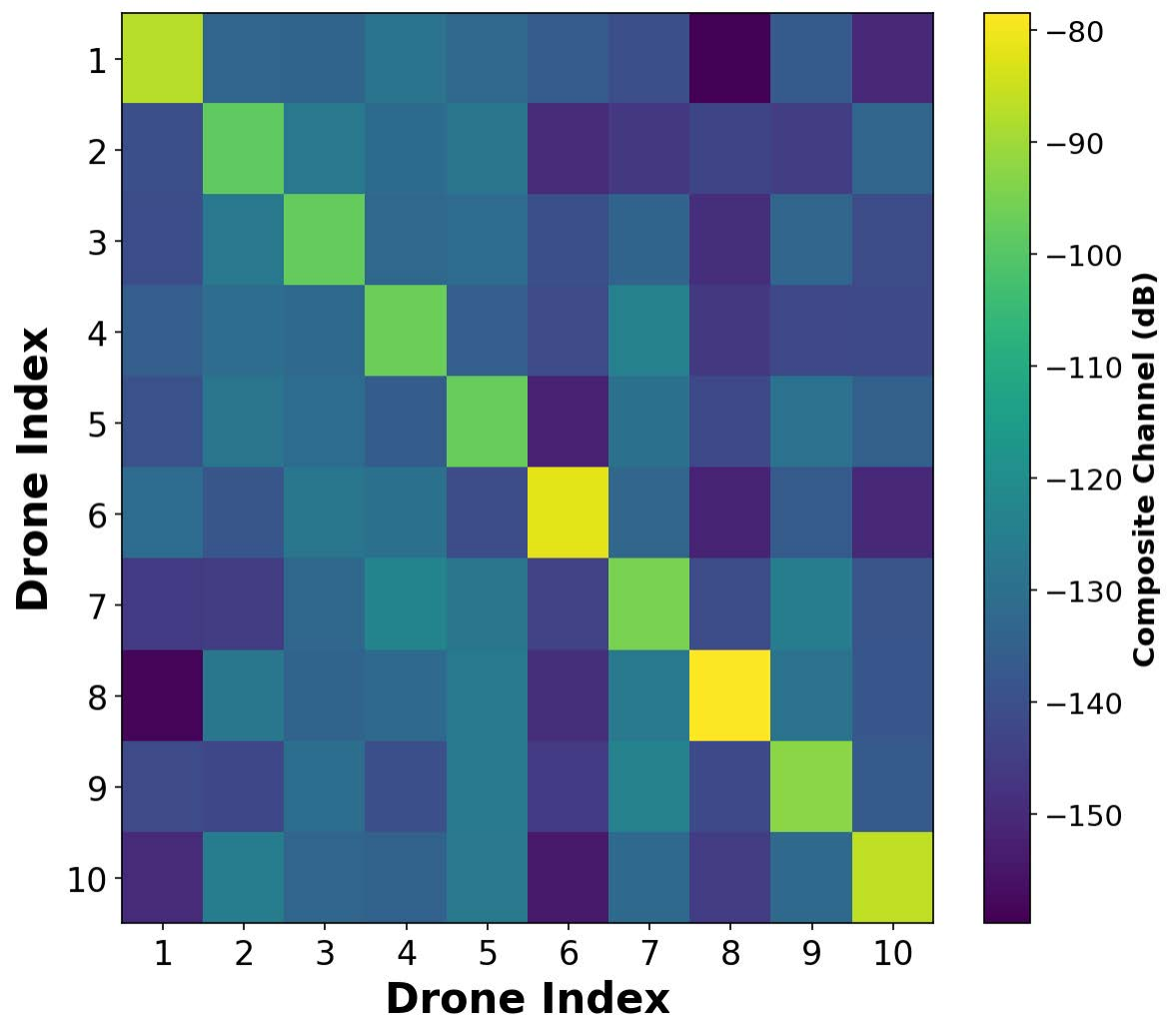}
		\caption{IRC channel map.}
	\end{subfigure}
	\begin{subfigure}{0.26\linewidth}
		\centering
		\includegraphics[width=1\linewidth]{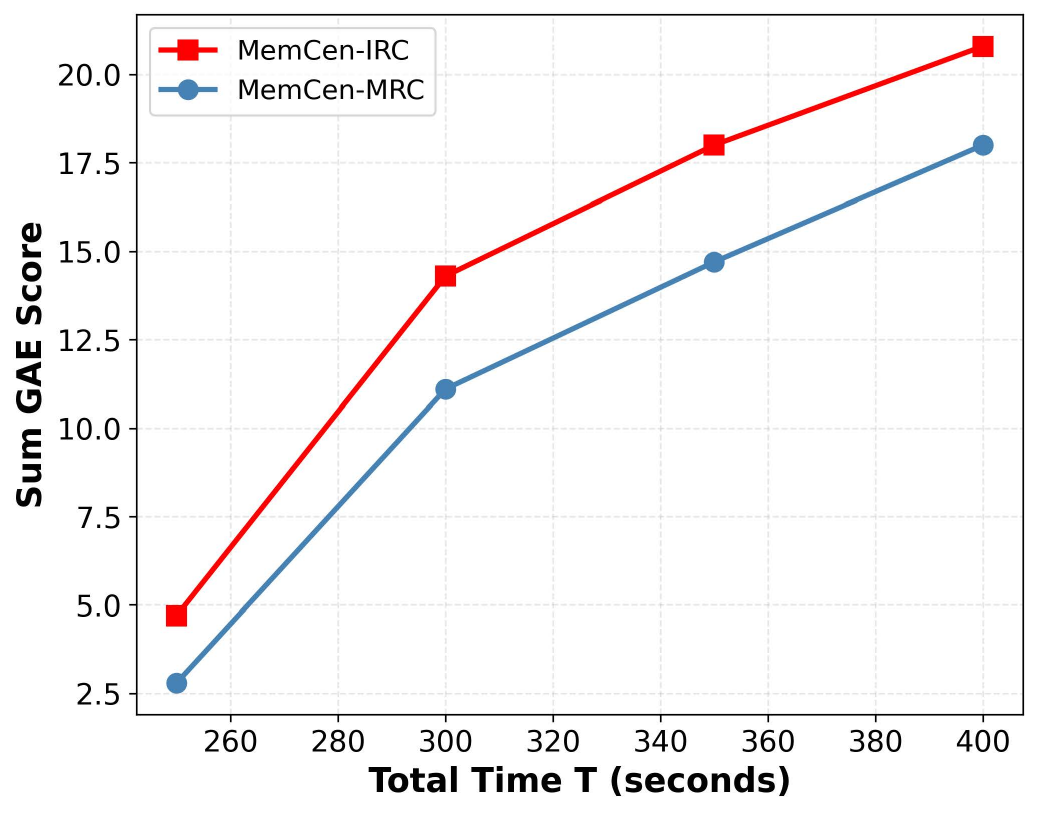}
		\caption{Sum GAE score versus $T$.}
	\end{subfigure}
     \begin{subfigure}{0.26\linewidth}
		\centering
		\includegraphics[width=1\linewidth]{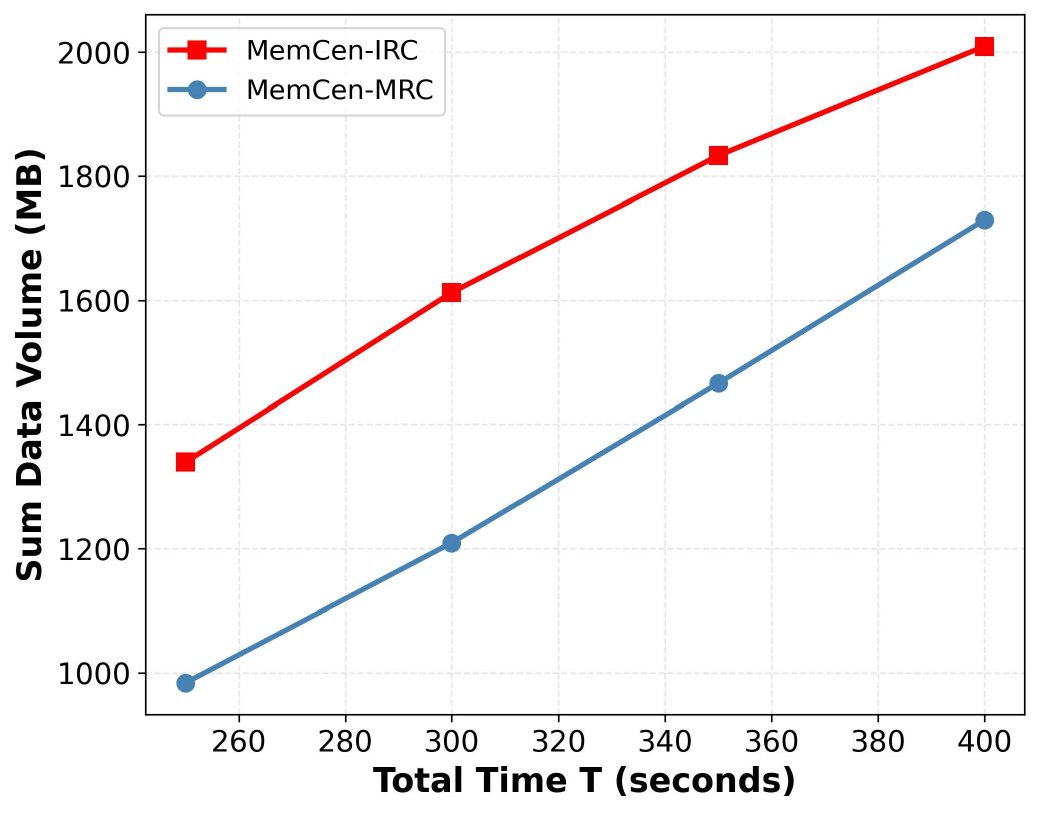}
		\caption{Data volume versus $T$.}
	\end{subfigure} 
    \caption{Comparison between MRC and IRC.}
	\label{fig:fig10}
\end{figure*}

We next evaluate the complete MemCen framework against the baselines in Town04.
We consider a scenario with $K=10$ UAVs and a total power budget of $P_{\mathrm{sum}}=300$\,mW, as illustrated in Fig.~\ref{fig:fig6}.
Among the 10 UAVs, five are randomly selected to construct the initial ground memory $\mathcal{M}_0$.
For GAE, we generate $L_k=10$ questions for each UAV from 10 randomly sampled images.
The quantitative results are reported in Table~\ref{tab:tab3}.

MemCen achieves the highest QA accuracy of 92.4\% and the highest GAE score of 20.7.
Compared with the standalone-memory baseline Remember \cite{anwar2025remembr}, MemCen improves QA accuracy by more than 50\%, highlighting the importance of memory aggregation for city-scale QA.
Although ComCen achieves the highest communication rate, its QA accuracy is 26.8 percentage points lower than that of MemCen.
This result shows that higher communication throughput does not necessarily translate into greater memory utility, particularly when nearby UAVs upload redundant observations that contribute little new knowledge to the ground server.
Similarly, SenCen performs worse despite achieving the largest sensing coverage in Town04, i.e., $8.1\times$ unit inspection areas.
This is because the value of additional sensing depends not only on coverage, but also on the information density and freshness of the observed regions.
Together, these results demonstrate that MemCen effectively identifies UAV memories that provide greater downstream QA utility.
Among the baselines, Greedy achieves the performance closest to MemCen.
Greedy prioritizes UAVs with larger GAE scores but does not jointly account for their communication costs and delivery feasibility.
As a result, it may allocate excessive resources to a few high-scoring UAVs, reducing the number of successfully connected UAVs and making the resulting memory aggregation less robust to discrepancies between GAE-estimated and actual memory utility.

Fig.~\ref{fig:fig8}(a) illustrates the accuracy--rate tradeoff in the ten-UAV Town04 scenario.
The resulting boundary reveals the key cross-layer tradeoff between communication-efficient memories delivered over strong links and high-value memories associated with weaker links.
The formulation $\mathsf{Q}$ makes this operating tradeoff explicit through the dimensionless preference parameter $\mu$.

Fig.~\ref{fig:fig8}(b)--(d) evaluates the impact of the total power budget by varying
$P_{\mathrm{sum}}\in\{100,200,300,400\}$\,mW.
MemCen consistently achieves the highest sum GAE score across all tested power budgets.
Its memory utility and number of admitted UAVs gradually saturate beyond $300$~mW, as most feasible high-value memories have already been acquired.
At $400$~mW, ComCen achieves nearly four times the sum rate of MemCen but does not attain comparable memory utility.
This result highlights that higher communication throughput does not necessarily translate into more task-relevant information.

Fig.~\ref{fig:fig9} and Table~\ref{tab:tab4} further examine a representative allocation with the initial memory set $\{1,2,5,7,10\}$.
ComCen prioritizes UAVs with strong communication links but may select memories with limited task utility.
SenCen admits more UAVs but excludes weak-link UAV~3 despite its rare task-relevant observations.
In contrast, MemCen selects UAV~3 while maintaining six successful memory uploads by jointly balancing memory utility and delivery cost.
Greedy considers memory utility alone, whereas FairCen fails to complete any full-memory upload under the considered constraints.
This case further demonstrates the importance of jointly optimizing memory utility, payload size, and communication cost.

We finally compare MRC and IRC at $\sigma^2=-90$\,dBm over
$T\in\{250,300,350,400\}$\,s.
Under this setting, the system operates in an interference-limited regime, where aggregate multi-UAV interference dominates the background noise.
As shown in Fig.~\ref{fig:fig10}(a)--(b), IRC preserves the desired-link gain while substantially suppressing off-diagonal interference leakage.
Consequently, IRC achieves higher memory utility and a larger successfully delivered data volume than MRC, as shown in Fig.~\ref{fig:fig10}(c)--(d).
These results support the operating conditions discussed in Remark~2 and motivate joint receiver, coding, and resource-allocation design as an interesting direction for future work.

\begin{figure*}[!t]
    \centering
    \begin{subfigure}[t]{0.57\textwidth}
        \centering
        \includegraphics[height=2.35in,keepaspectratio]{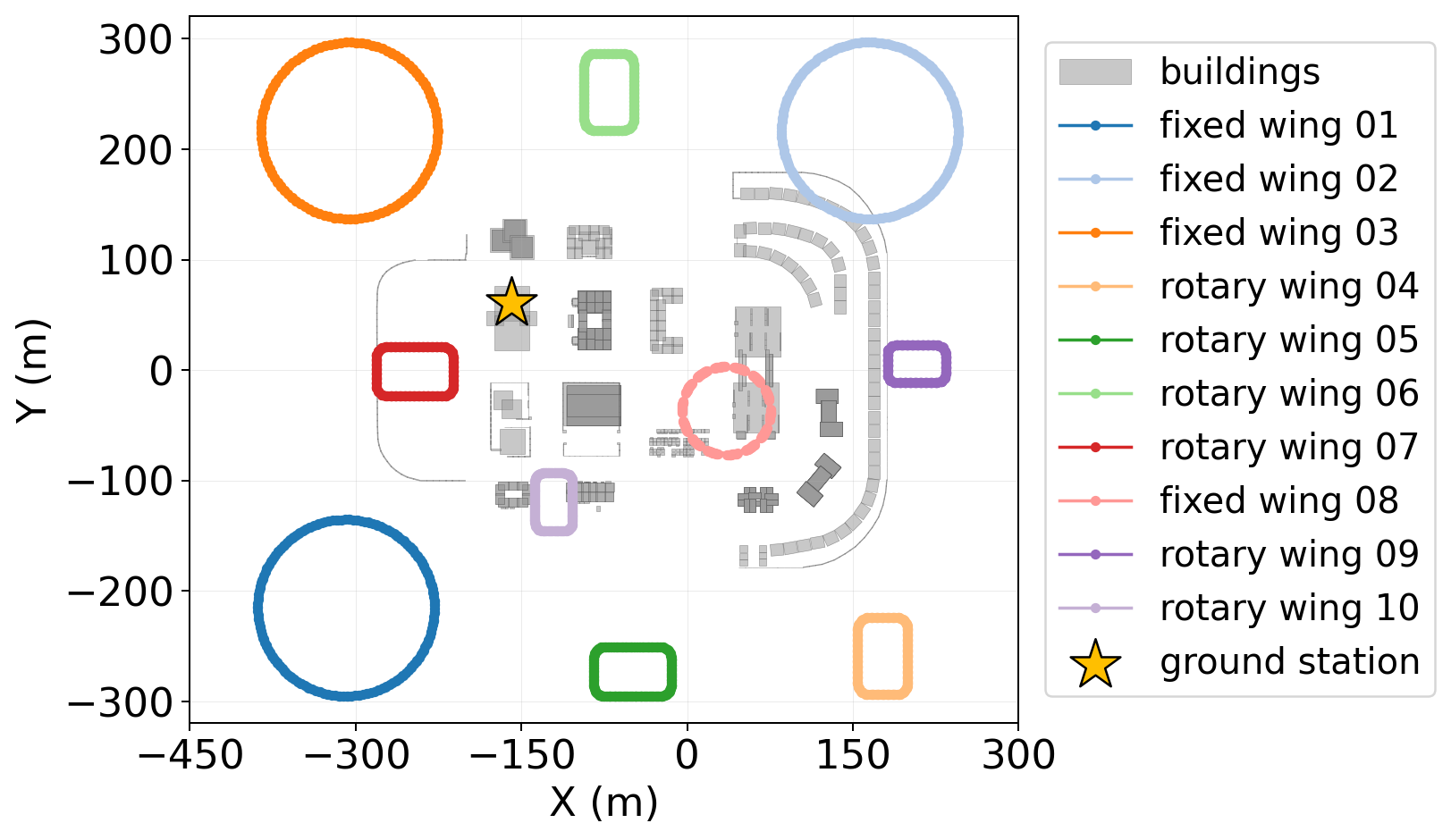}
        \caption{{Town05 geometry, ground station, and UAV trajectories.}}
    \end{subfigure}\hfill
    \begin{subfigure}[t]{0.42\textwidth}
        \centering
        \includegraphics[height=2.35in,keepaspectratio]{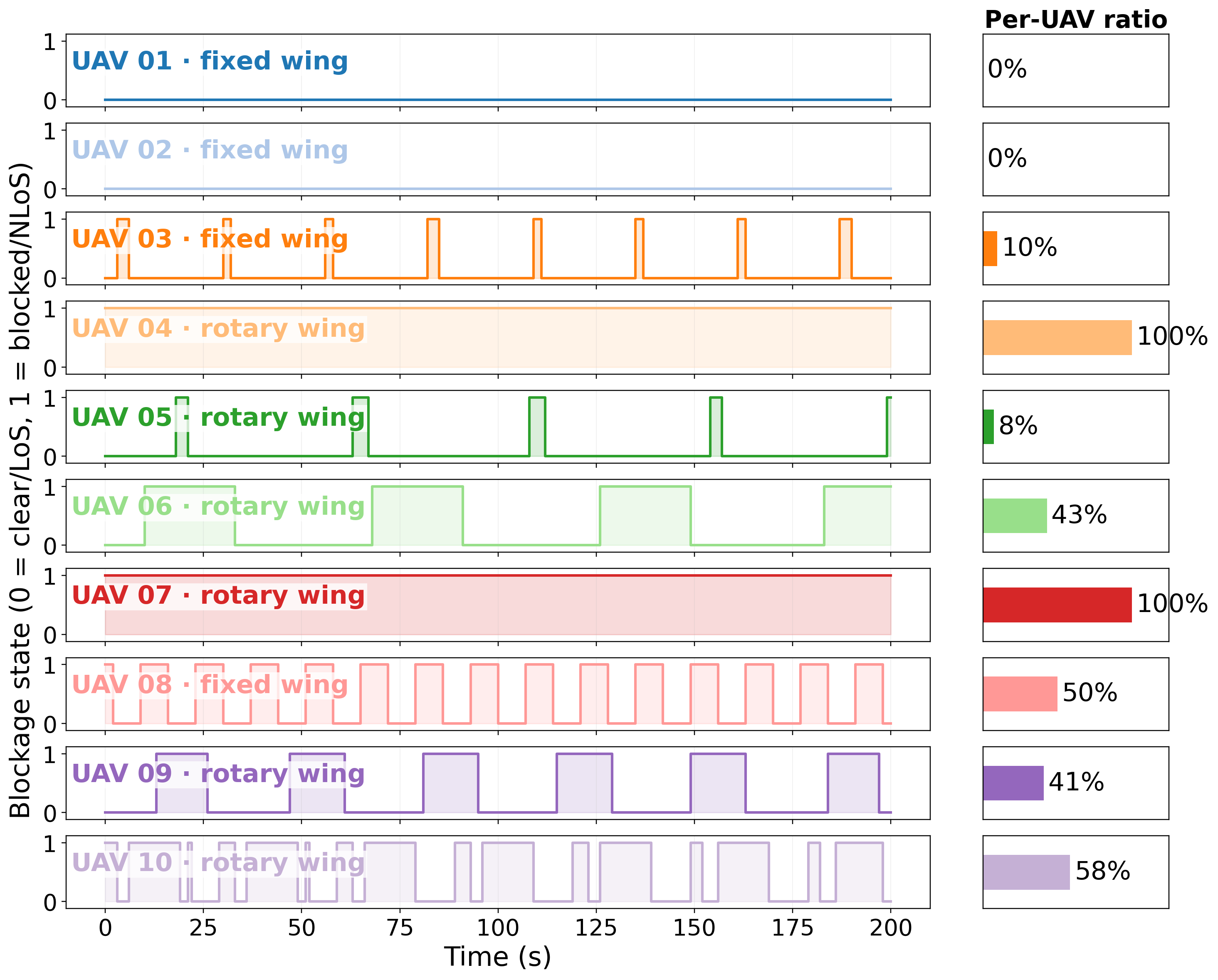}
        \caption{{Time-varying blockage across the ten UAV links.}}
    \end{subfigure}
    \par\medskip
    \begin{subfigure}[t]{0.245\textwidth}
        \includegraphics[width=\linewidth]{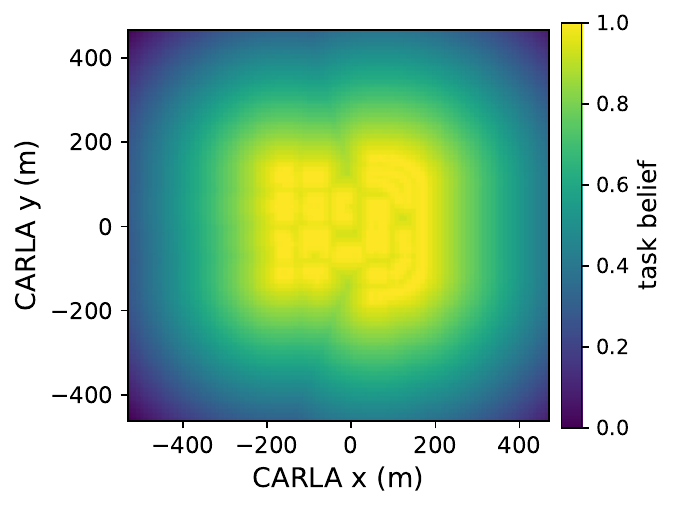}
        \caption{{Urban-inspection belief.}}
    \end{subfigure}\hfill
    \begin{subfigure}[t]{0.245\textwidth}
        \includegraphics[width=\linewidth]{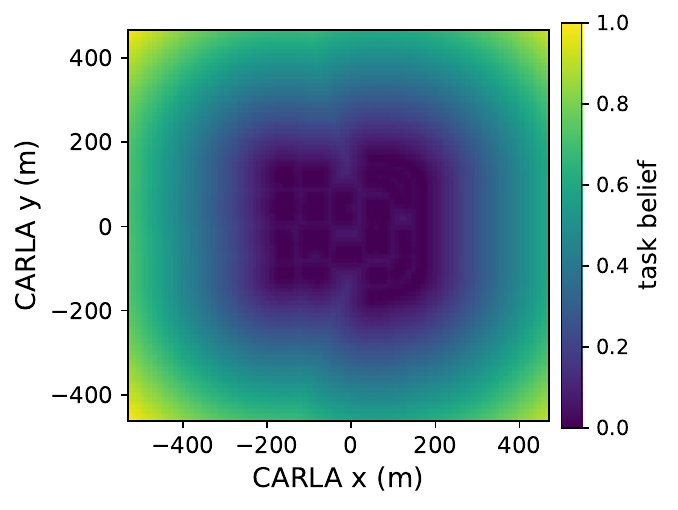}
        \caption{{Suburban-rescue belief.}}
    \end{subfigure}\hfill
    \begin{subfigure}[t]{0.245\textwidth}
        \includegraphics[width=\linewidth]{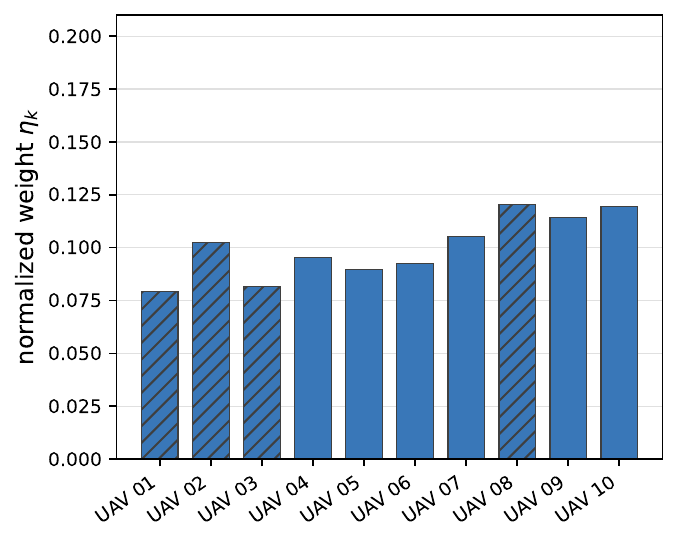}
        \caption{{Inspection route weights.}}
    \end{subfigure}\hfill
    \begin{subfigure}[t]{0.245\textwidth}
        \includegraphics[width=\linewidth]{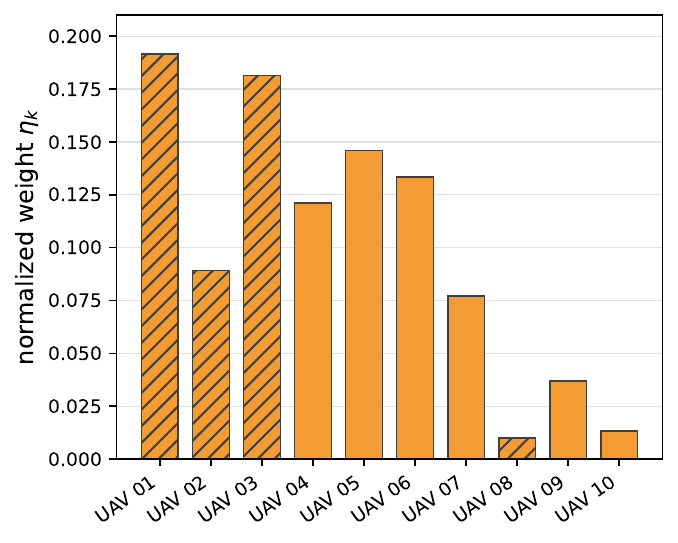}
        \caption{{Rescue route weights.}}
    \end{subfigure}
    \caption{{Dynamic Town05 setting.
    (a)--(b) show the trajectories and geometry-induced channel blockage. (c)--(f) show task switch on the same map and routes. Hatched and solid bars denote fixed-wing and multirotor profiles, respectively.}}
    \label{fig:fig11}
\end{figure*}

\begin{table*}[!t]
    \centering
    \caption{{Town05 results over 20 random trials.}}
    \label{tab:tab5}
    \setlength{\tabcolsep}{7pt}
    \renewcommand{\arraystretch}{1.08}
    \small
    {
    \begin{tabular}{lcccc}
        \toprule
        Method & $\uparrow$QA accuracy &
        $\uparrow$Sum QoM & Selected UAVs & $\uparrow$Sum rate (Mbps) \\
        \midrule
        Remember & $33.0\!\pm\!14.9\%$ & $0.000\!\pm\!0.000$ & $0.00\!\pm\!0.00$ & $0.00\!\pm\!0.00$ \\
        FairCen & $33.0\!\pm\!14.9\%$ & $0.000\!\pm\!0.000$ & $0.00\!\pm\!0.00$ & $31.06\!\pm\!0.29$ \\
        ComCen & $59.0\!\pm\!16.5\%$ & $1.505\!\pm\!0.544$ & $4.00\!\pm\!0.00$ & \maxval{\textbf{$369.10\!\pm\!0.29$}} \\
        SenCen & $71.0\!\pm\!15.2\%$ & $2.435\!\pm\!0.509$ & $7.00\!\pm\!0.00$ & $289.84\!\pm\!0.22$ \\
        SemCom & $79.0\!\pm\!16.5\%$ & $2.770\!\pm\!0.540$ & \maxval{$7.80\!\pm\!0.41$} & $230.18\!\pm\!25.07$ \\
        \textbf{MemCen} & \maxval{\textbf{$84.0\!\pm\!10.5\%$}} & \maxval{\textbf{$2.775\!\pm\!0.454$}} & $6.70\!\pm\!0.73$ & $285.91\!\pm\!0.29$ \\
        \bottomrule
    \end{tabular}}
\end{table*}

\begin{table*}[!t]
    \centering
    \caption{Sensitivity of MemCen to image loads in Town05. FW and MR denote
    fixed-wing and multirotor UAVs.}
    \label{tab:tab6}
    \setlength{\tabcolsep}{4.5pt}
    \renewcommand{\arraystretch}{1.08}
    \footnotesize
    \begin{tabular}{ccccccc}
        \toprule
        \makecell{FW frames\\per UAV} & \makecell{MR frames\\per UAV} &
        \makecell{Mean frame size\\FW/MR (MB)} &
        {$\uparrow$Sum rate (Mbps)} & Selected UAVs &
        $\uparrow$Sum QoM & {$\uparrow$QA accuracy} \\
        \midrule
        1,000 & 1,000 & 0.666/0.286 & $184.87\!\pm\!0.03$ & \maxval{$9.00\!\pm\!0.00$} & \maxval{$3.390\!\pm\!0.409$} & \maxval{$93.0\!\pm\!9.8\%$} \\
        1,500 & 1,000 & 0.667/0.286 & $244.66\!\pm\!4.14$ & $8.45\!\pm\!0.83$ & $3.185\!\pm\!0.627$ & $89.0\!\pm\!12.1\%$ \\
        2,000 & 1,000 & 0.666/0.286 & \maxval{$285.91\!\pm\!0.29$} & $6.70\!\pm\!0.73$ & $2.775\!\pm\!0.454$ & $84.0\!\pm\!10.5\%$ \\
        \bottomrule
    \end{tabular}
\end{table*}

\subsection{Dynamic and Heterogeneous Town05 Evaluation}

We next test whether MemCen remains effective when mobility,
blockage, and sensing workloads vary jointly. 
As shown in Fig.~\ref{fig:fig11}(a)--(b), we consider a 200-s search-and-rescue mission in CARLA Town05 with $K=10$ UAVs. 
Four fixed-wing UAVs capture images at a resolution of $1920\times1080$ pixels, whereas the six multirotor UAVs use $1280\times720$ pixels. The average image sizes are approximately $666.5$\,kB and $286.0$\,kB, respectively.
Their effective workloads are $2,000$ and $1,000$ frames per UAV, respectively. 
The server is placed on a rooftop, and the channel is computed using UAV trajectories. 
Town05 building geometry yields a 40.9\% blockage ratio, and each blocked link incurs an additional 20-dB attenuation.

MemCen operates at two time scales. Before flight, the server selects the
memory set using the task-conditioned QoM and the channel statistics. During flight, it adapts transmit power using instantaneous
CSI while keeping the selected set fixed. A task change triggers a new
slow-time-scale selection. We use $B=10$~MHz, $P_{\rm sum}=0.3$~W,
$P_{\max}=0.1$~W, and a noise power of $-100$~dBm. Results are averaged over
20 random trials. In each trial, three randomly sampled UAV memories form
$\mathcal M_0$, and five rescue targets are assigned to
five distinct suburban UAVs. The initial set, target assignment, and fading
realization vary across trials. For SemCom, 25 evenly spaced frames per UAV are
ranked by their cosine dissimilarity from $\mathcal M_0$.

Fig.~\ref{fig:fig11}(c)--(f) instantiates the belief-map
interface introduced in Section~\ref{section3}. Let
$B^{\rm res}(i,j)$ be the normalized distance from grid cell $(i,j)$ to its
nearest building and let $B^{\rm ins}(i,j)=1-B^{\rm res}(i,j)$. The former
emphasizes suburban rescue areas, whereas the latter emphasizes the urban
building core. We integrate the active map along each recorded route and
normalize the resulting ten scores to obtain $\eta_k$. Switching from urban
inspection to suburban rescue decreases the aggregate weight of the three
urban routes from 0.354 to 0.060 and shifts the prior toward the seven suburban
routes. This controlled analysis validates that the same scene and memory
interface can be reweighted without platform-specific retraining.

Table~\ref{tab:tab5} shows that MemCen achieves the highest QA accuracy of 84.0\% in Town05.
SenCen admits fewer UAVs than SemCom because some of its selected UAVs fail to complete their memory uploads under the dynamic channels.
In contrast, MemCen successfully completes all selected uploads and reduces the standard deviation of QA accuracy from 16.5\% to 10.5\%.
Its selected UAV set contains $1.70$ fixed-wing UAVs and $5.00$ multirotor UAVs on average.
This result shows that MemCen can jointly accommodate the two heterogeneous UAV platforms.
Overall, these results demonstrate robust memory selection and resource allocation under heterogeneous payloads, UAV mobility, and building blockage.

Table~\ref{tab:tab6} evaluates the impact of nonuniform image volume by increasing each fixed-wing workload from 1,000 to 2,000 frames while fixing each multirotor workload at 1,000 frames. 
As the fixed-wing workload increases, MemCen becomes more selective. The number of admitted UAVs decreases from 9.00 to 6.70, the aggregate QoM decreases from 3.390 to 2.775, and the QA accuracy drops from 93.0\% to 84.0\%.
Despite the heavier communication load, all selected memory uploads remain feasible.
The higher sum rate observed under larger workloads results from admitting fewer UAVs, which reduces multiuser interference among the active links.

\section{Real-World Experiment}\label{section6}

\subsection{Physical Multi-UAV Field Evaluation}

We build a panoramic multi-agent system (PMAS) with three DJI AVATA 360 UAVs flying distinct 200-s routes in the same
outdoor scene. Fig.~\ref{fig:fig12} shows the experimental setup,
UAV trajectories, and representative observations.
Each original 10-Hz video contains 2,000 frames.
Panoramic capture supports spatial intelligence by providing broad
visual coverage as each UAV changes heading. This reduces blind spots and
preserves context for spatial grounding. We use COLMAP-based panoramic
structure from motion (SfM) to register observations and trajectories in a
shared 3D coordinate frame.
Unless otherwise specified, we use \texttt{Qwen3-VL-8B} for image captioning, \texttt{Mxbai-Embed-Large-v1} for text embedding, and \texttt{Milvus} for top-5 memory retrieval. 
\texttt{Qwen3-30B-Instruct} is used to generate evidence-grounded answers.

We conduct 20 random trials. Each trial samples the initial memory and an independent channel realization, then evaluates ten object-presence and ten
spatial-grounding questions. The aerial data are collected in the field, while communication is evaluated through offline channel replay.
The ground server is placed at $(0,35)$ in the COLMAP coordinates. We use $B=5$ MHz, $N=8$, $\sigma^2=-100$ dBm,
$P_{\mathrm{sum}}=100$ mW, and $T=200$ s. Each link involves a 20-dB blockage
loss.

\begin{figure*}[!t]
    \centering
    \includegraphics[width=0.96\textwidth]{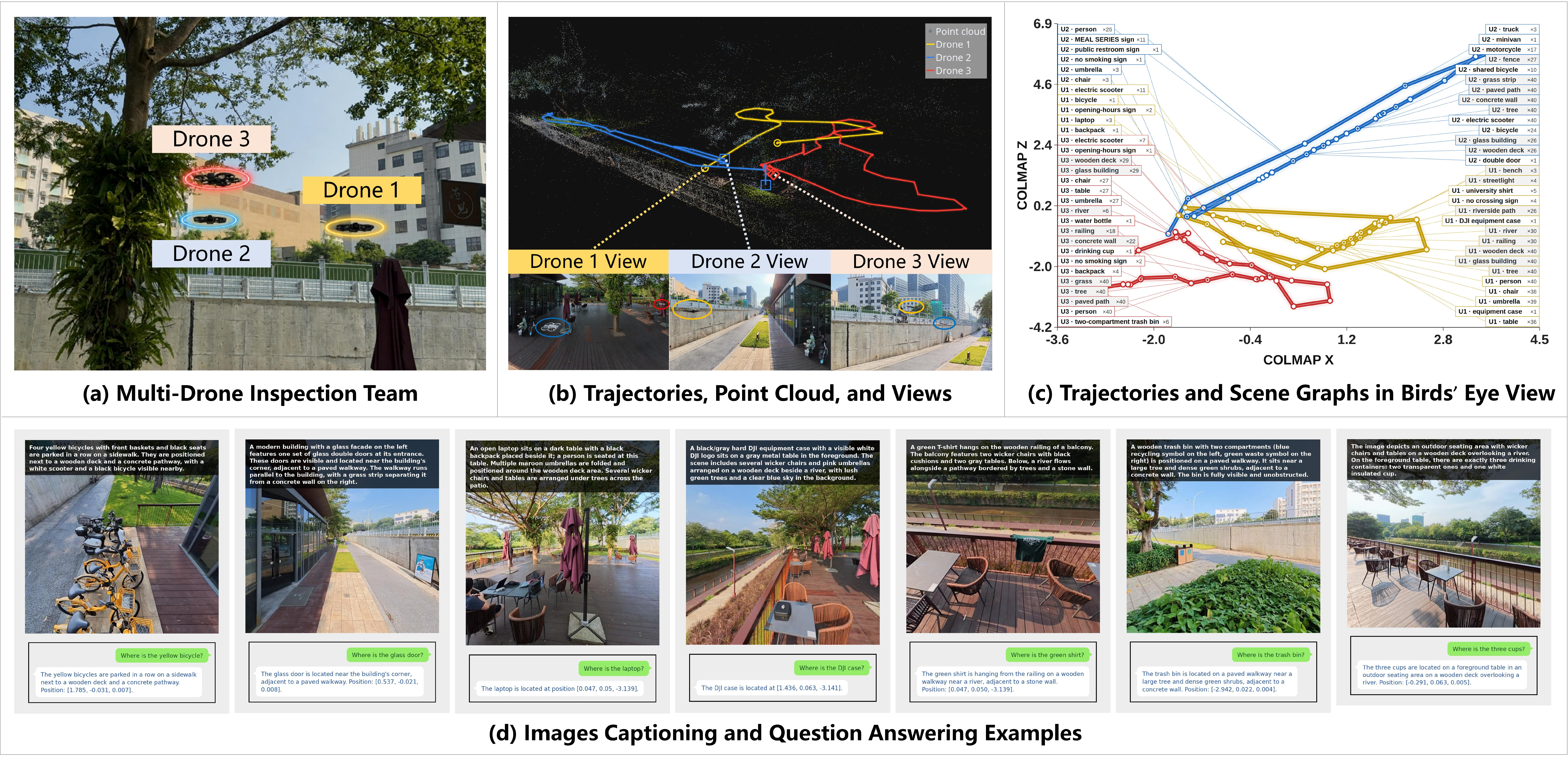}
    \caption{Real-world PMAS benchmark. The figure shows
    simultaneous operation, reconstructed trajectories, representative onboard
    views, semantic memories, and evidence-grounded spatial QA.}
    \label{fig:fig12}
\end{figure*}

MemCen improves the mean QA accuracy by at least 12.0 percentage points over all baselines in Table~\ref{tab:tab7}.
Although MemCen and SenCen request the same average number of uploads, i.e., 1.90, MemCen achieves a 54.7\% higher sum QoM.
This result confirms that the performance gain primarily arises from selecting more valuable memories rather than simply uploading more data.
The relatively lower accuracy on ``Where'' questions further indicates that spatial grounding remains the dominant failure mode.

\begin{figure*}[!t]
    \centering
    \includegraphics[width=0.98\textwidth]{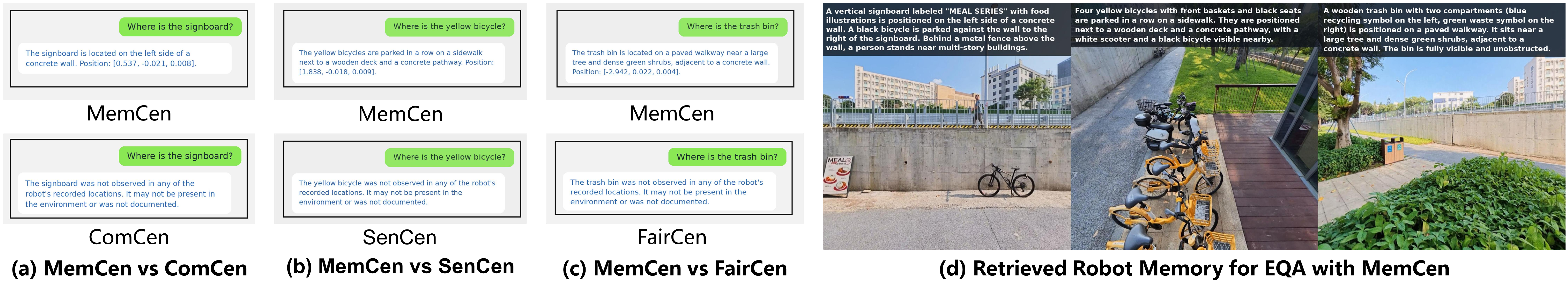}
    \caption{{Case-level comparison of spatially grounded QA
    and the visual evidence retrieved by MemCen.}}
    \label{fig:fig13}
\end{figure*}

\begin{figure*}[!t]
    \centering
    \begin{subfigure}[t]{0.34\linewidth}
        \centering
        \vspace{0pt}
        \includegraphics[width=\linewidth]{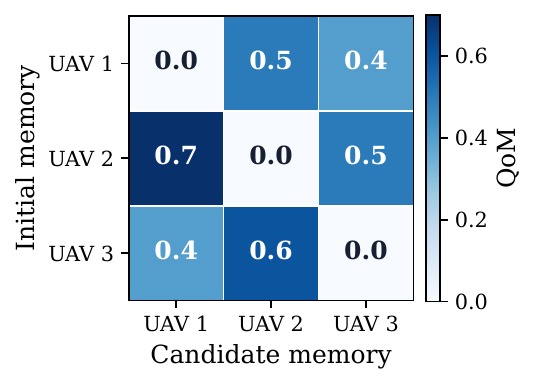}
        \caption{Marginal QoM matrix.}
    \end{subfigure}\hfill
    \begin{subfigure}[t]{0.32\linewidth}
        \centering
        \vspace{0pt}
        \includegraphics[width=\linewidth]{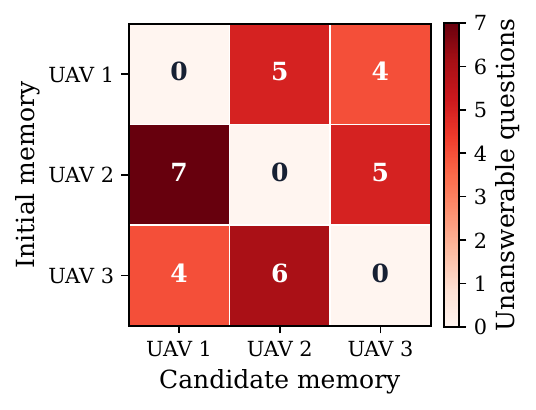}
        \caption{GAE scores.}
    \end{subfigure}\hfill
    \begin{subfigure}[t]{0.33\linewidth}
        \centering
        \vspace{0pt}
        \includegraphics[width=\linewidth]{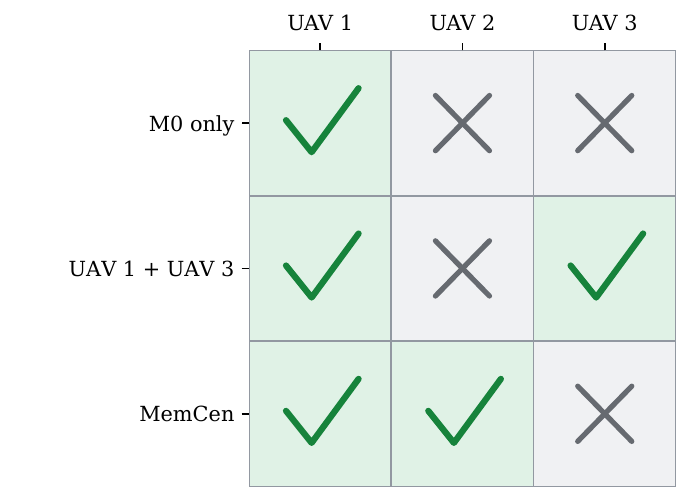}
        \caption{One-upload redundancy ablation.}
    \end{subfigure}
    \caption{QoM and GAE matrices for each initial memory, and a case study with $\mathcal M_0=\mathcal M_1$.}
    \label{fig:fig14}
    \vspace{-0.1in}
\end{figure*}

\begin{table*}[!t]
    \centering
    \caption{Quantitative results over 20 random trials on the PMAS data.}
    \label{tab:tab7}
    \setlength{\tabcolsep}{5.5pt}
    \renewcommand{\arraystretch}{1.08}
    \small
    {
    \begin{tabular}{lcccccc}
        \toprule
        Method & $\uparrow$QA accuracy & $\uparrow$YES/NO & $\uparrow$Where &
        $\uparrow$Sum QoM & Avg. uploads & $\uparrow$Sum rate \\
        \midrule
        Remember & $45.25\!\pm\!4.32\%$ & 50.5\% & 40.0\% & 0.000 & 0.00 & -- \\
        SenCen & $76.50\!\pm\!8.67\%$ & 83.5\% & 69.5\% & 0.640 & \maxval{1.90} & 36.53 Mbps \\
        ComCen & $61.50\!\pm\!12.05\%$ & 71.0\% & 52.0\% & 0.320 & 1.15 & \maxval{\textbf{36.68} Mbps} \\
        FairCen & $62.00\!\pm\!12.19\%$ & 72.5\% & 51.5\% & 0.310 & 1.30 & 36.63 Mbps \\
        \textbf{MemCen} & \maxval{\textbf{$88.50\!\pm\!4.50\%$}} & \maxval{\textbf{98.0\%}} &
        \maxval{\textbf{79.0\%}} & \maxval{\textbf{0.990}} & \maxval{1.90} & 25.65 Mbps \\
        \bottomrule
    \end{tabular}}
    \vspace{-0.2in}
\end{table*}

We next examine a representative trial with $\mathcal{M}_0$ constructed from UAV~3.
As shown in Fig.~\ref{fig:fig13}, the selected memories provide sufficient evidence to answer all three spatial queries.
For the yellow-bicycle query, the estimated position has an error of only 0.054~m, whereas the compared baselines omit the required memory and return an unknown answer.
Fig.~\ref{fig:fig14}(a) shows QoM values of 0.4, 0.6, and 0 for UAVs 1--3, respectively, while Fig.~\ref{fig:fig14}(b) shows the corresponding unanswered-question counts of 4, 6, and 0.
MemCen selects UAVs 1 and 2, yielding an aggregate QoM of 1.0.
This representative trial illustrates the complete MemCen pipeline from GAE-estimated memory utility and physical-layer resource allocation to memory retrieval and downstream QA performance.

To examine whether MemCen can effectively handle redundant observations across different UAV viewpoints, we conduct a controlled case study. 
We set $\mathcal{M}_0$ to the memory of UAV~1 and allow one additional UAV upload.
UAV~3 exhibits 7.7\% trajectory overlap and 41.2\% semantic overlap with UAV~1, resulting in a marginal QoM of 0.40.
In contrast, UAV~2 has lower trajectory and semantic overlaps of 6.4\% and 23.8\%, respectively, and therefore achieves a higher marginal QoM of 0.50.
As shown in Fig.~\ref{fig:fig14}(c), MemCen selects UAV~2.
The resulting QA accuracy reaches 75\%, compared with 70\% for the combination of UAV~1 and UAV~3 and 50\% for $\mathcal{M}_0$ alone.
This case demonstrates that GAE naturally discounts observations that are redundant with respect to different viewpoints, thereby favoring candidate memories that contribute more complementary information.

\begin{figure}[!t]
    \centering
    \includegraphics[width=0.47\textwidth]{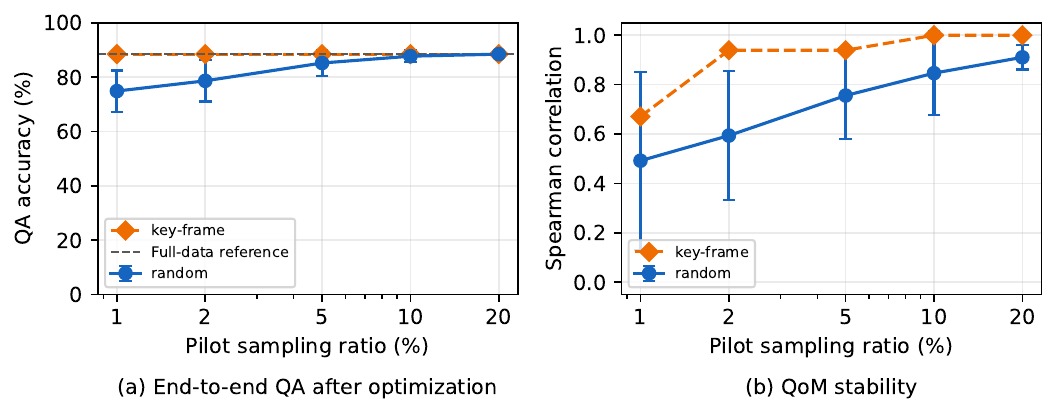}
    \caption{{QoM stability and downstream QA versus the pilot
    sampling ratio. Error bars show the standard deviation of random sampling.
    Dashed lines show the full-data reference.}}
    \label{fig:fig15}
    \vspace{-0.1in}
\end{figure}

We then evaluate pilot ratios $\rho_k\in\{1\%,2\%,5\%,10\%,20\%\}$ using the PMAS data.
A pilot setting is considered reliable if its mean Spearman correlation with the full-data QoM ranking is at least $0.8$ and its mean downstream QA regret is no more than five percentage points.
As shown in Fig.~\ref{fig:fig15}, random sampling satisfies both criteria at a pilot ratio of $10\%$, achieving a Spearman correlation of $0.847$ and a QA accuracy of $87.75\%$.
In contrast, key-frame sampling satisfies both criteria at only $2\%$, with a Spearman correlation of $0.939$ and a QA accuracy of $88.50\%$.
No error bar is reported for key-frame sampling because the selection is deterministic.
Thus, the minimum observed reliable pilot ratios are $10\%$ for random sampling and $2\%$ for key-frame sampling under the evaluated setting.
These thresholds are specific to the tested scene and frame rate.
Accordingly, we adopt $10\%$ random sampling when onboard key-frame extraction is unavailable.

Table~\ref{tab:tab8} evaluates the sensitivity of QoM estimation and memory selection to model choice.
Within each comparison, the input images, questions, and top-5 retrieval setting are kept fixed.
First, increasing the model size generally improves QA accuracy, reflecting the stronger reasoning capability of larger models.
Changing the VLM shifts the mean QoM from 0.450 to 0.567, while preserving the preferred candidate memory in all three initial-memory cases.
With captions fixed, replacing \texttt{Mxbai-Embed-Large-v1} with \texttt{BGE-Large-EN-v1.5} likewise preserves all three selections when using the 4B and 8B answering models.
For the 30B answering model, the embedding change results in one tie and alters one selection.
These results show that the absolute QoM values are model dependent, whereas the resulting memory selections remain stable in most tested settings and are mainly sensitive when candidate memories have similar QoM scores.

\begin{table}[!t]
    \centering
    \caption{Model sensitivity on the real-world data.}
    \label{tab:tab8}
    \small
    \setlength{\tabcolsep}{3pt}
    \renewcommand{\arraystretch}{1.08}
    \begin{tabular*}{\columnwidth}{@{\extracolsep{\fill}}ccccc@{}}
        \toprule
        VLM & LLM & Embedding & QA accuracy & Mean QoM \\
        \midrule
        \multicolumn{5}{c}{Comparison of different VLMs} \\ 
        \multicolumn{5}{c}{(tested on the 20-question benchmark)} \\
        \midrule
        4B  & 30B & Mxbai & 80.00\%  & 0.500 \\
        8B  & 30B & Mxbai & 90.00\%  & 0.567 \\
        32B & 30B & Mxbai & 100.00\% & 0.450 \\
        \midrule
        \multicolumn{5}{c}{Comparison of different LLMs and embeddings} \\
        \multicolumn{5}{c}{(tested on 30 grounded source-memory questions)} \\
        \midrule
        8B & 4B  & Mxbai & 80.00\%  & 0.717 \\
        8B & 4B  & BGE   & 86.67\%  & 0.683 \\
        8B & 8B  & Mxbai & 90.00\%  & 0.650 \\
        8B & 8B  & BGE   & 96.67\%  & 0.617 \\
        8B & 30B & Mxbai & 96.67\%  & 0.517 \\
        8B & 30B & BGE   & 100.00\% & 0.383 \\
        \bottomrule
    \end{tabular*}
    \vspace{-0.1in}
\end{table}

Note that top-5 retrieval results contain a supporting pose for 26 questions, corresponding to a support-hit rate of 86.7\% and a miss rate of 13.3\%.
Despite these retrieval misses, the complete top-5 QA pipeline correctly answers 29 of the 30 questions.
This difference indicates that final answer correctness does not always require retrieval of the annotated supporting pose, as alternative retrieved evidence may still be sufficient.
In our pipeline-conditioned QoM evaluation, retrieval failures are nevertheless treated as part of the end-to-end system behavior.

Lastly, we profile the computational overhead of GAE on a single RTX~5090.
For the physical three-UAV setting with $L_k=5$ and one top-5 retrieval round, a single GAE evaluation requires $18$ local LLM calls and $15$ local retrieval operations.
It processes 20,778 input tokens and 2,257 output tokens, with an average latency of $29.96$\,s.
All models are executed locally, incurring no external API cost.
The model execution reaches a peak GPU memory usage of approximately $23.3$\,GB.
For a larger setting with $K=50$, GAE requires $300$ LLM calls and $250$ local retrieval operations.
Assuming $N_{\mathrm{inf}}=16$ concurrent inference slots, extrapolation from the serial measurements yields an inference latency of approximately $31.2$\,s.

\subsection{Drone-to-Dog Collaboration}

\begin{figure*}[!t]
    \centering
    \begin{subfigure}{0.26\linewidth}
        \centering
        \includegraphics[width=\linewidth]{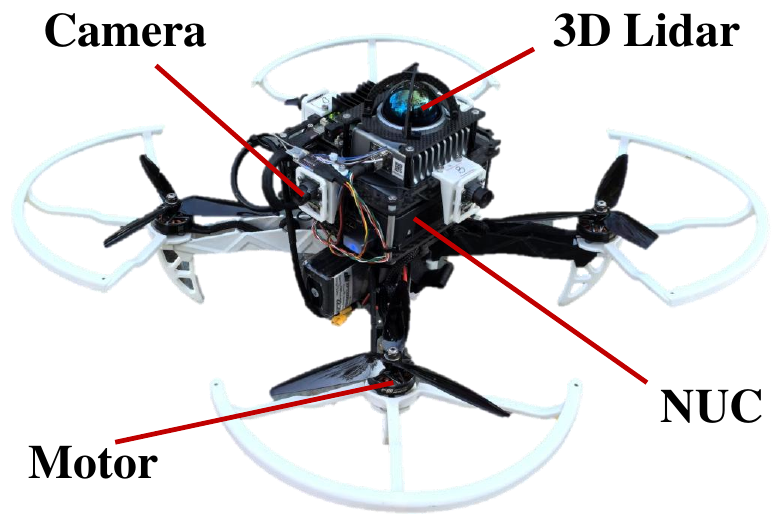}
        \caption{UAV.}
    \end{subfigure}
    \hfill
        \begin{subfigure}{0.18\linewidth}
        \centering
        \includegraphics[width=\linewidth]{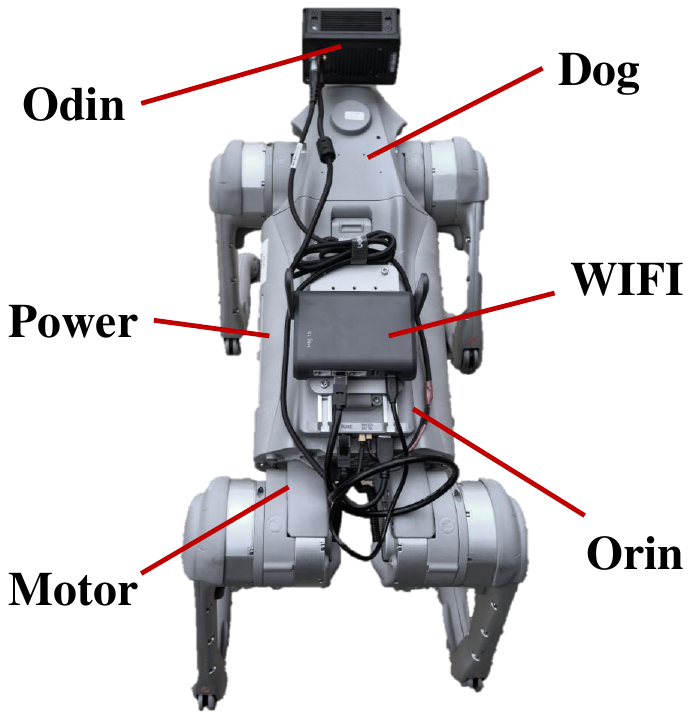}
        \caption{Dog.}
    \end{subfigure} 
    \hfill
        \begin{subfigure}{0.53\linewidth}
        \centering
        \includegraphics[width=\linewidth]{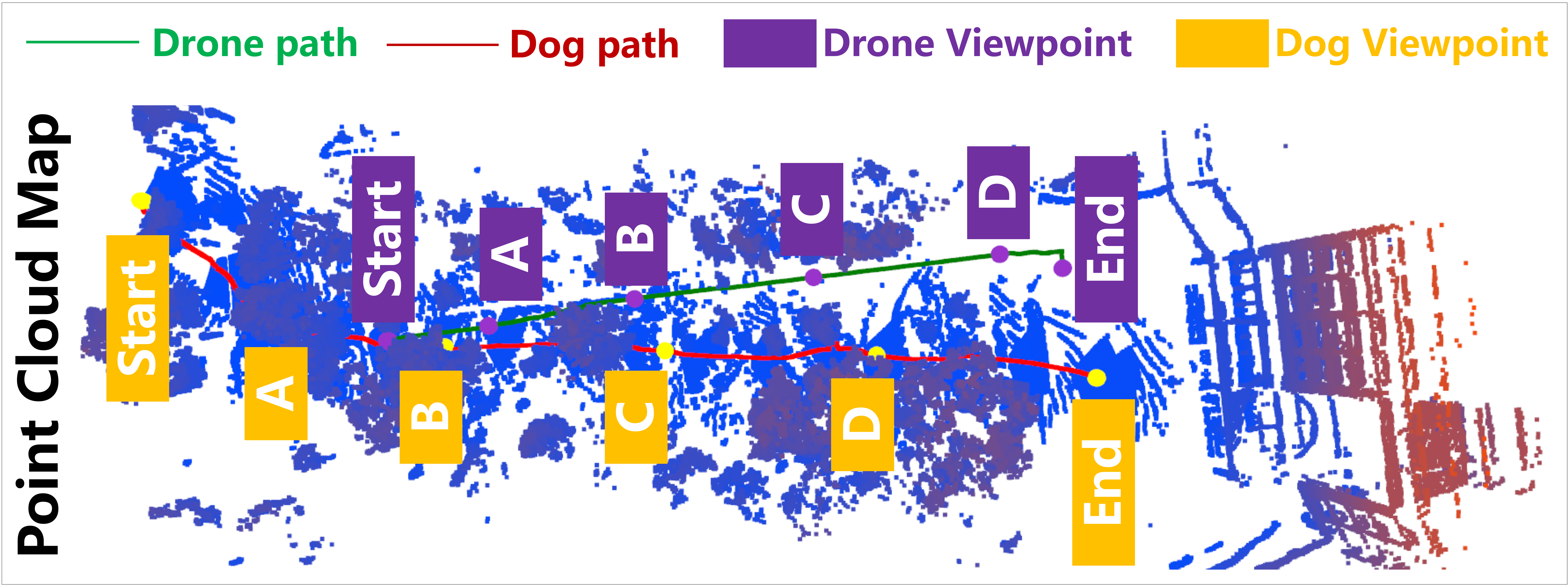}
        \caption{Point cloud map and UAV--ground-robot trajectories.}
    \end{subfigure} 
    \caption{Experimental setup of the real-world UAV-to-ground-robot collaboration task.}
    \label{fig:fig16}
    \vspace{-0.1in}
\end{figure*}

\begin{figure*}[t]
    \centering
    \includegraphics[width=0.95\textwidth]{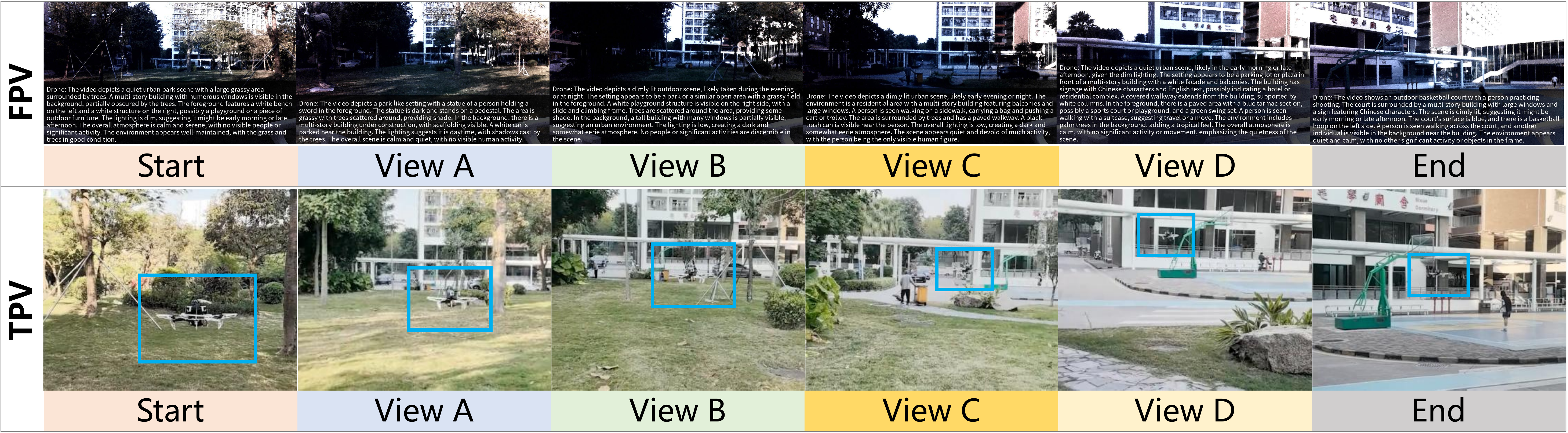}
    \caption{Single-UAV views. Their positions are marked by
    purple boxes in Fig.~\ref{fig:fig16}(c).}
    \label{fig:fig17}
    \vspace{-0.1in}
\end{figure*}

\begin{figure}[!t]
	\centering
	\begin{subfigure}{0.48\linewidth}
		\centering
		\includegraphics[width=1\linewidth]{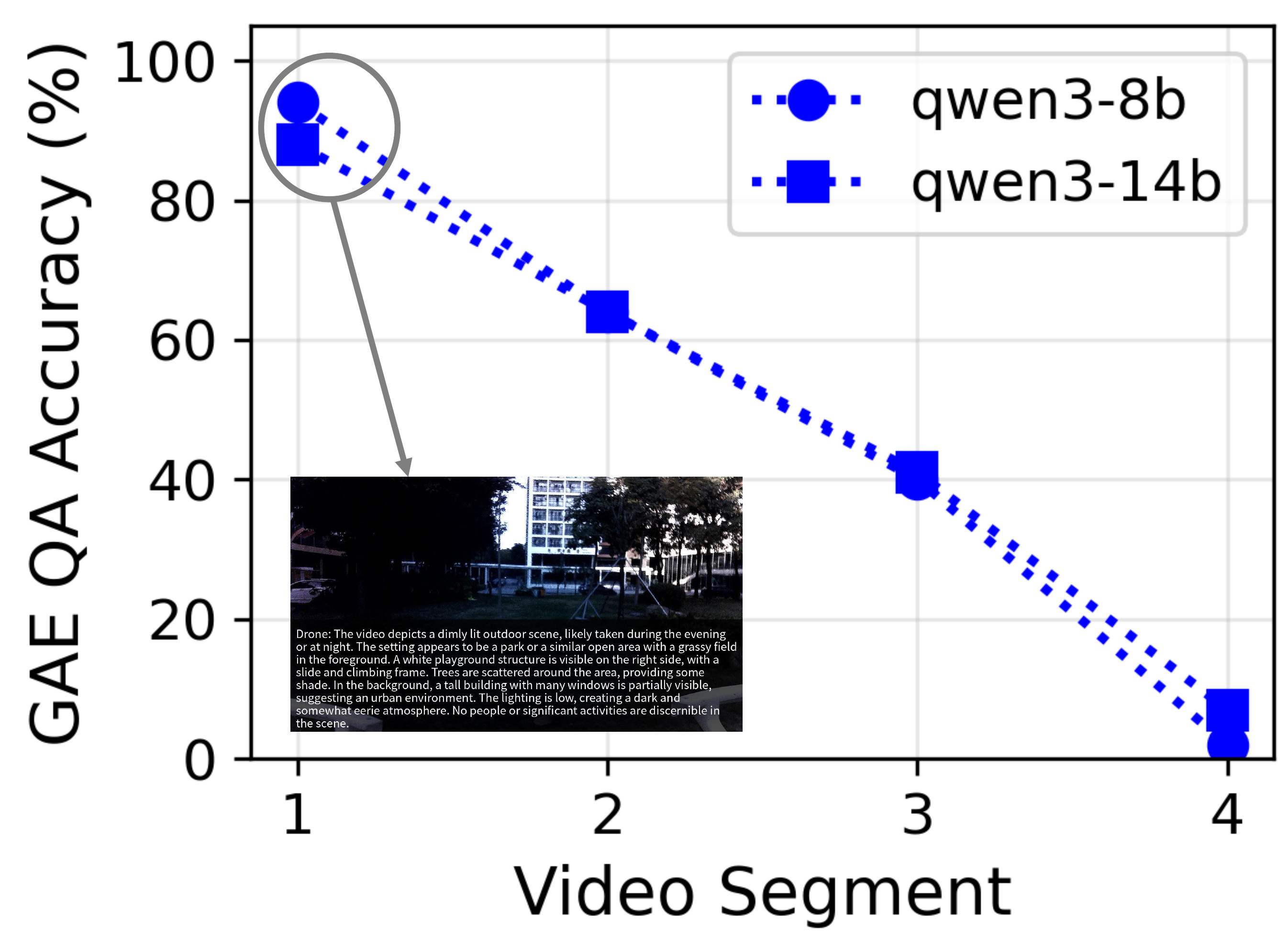}
		\caption{GAE accuracy vs. segment.}
	\end{subfigure}
 	\begin{subfigure}{0.48\linewidth}
		\centering
		\includegraphics[width=1\linewidth]{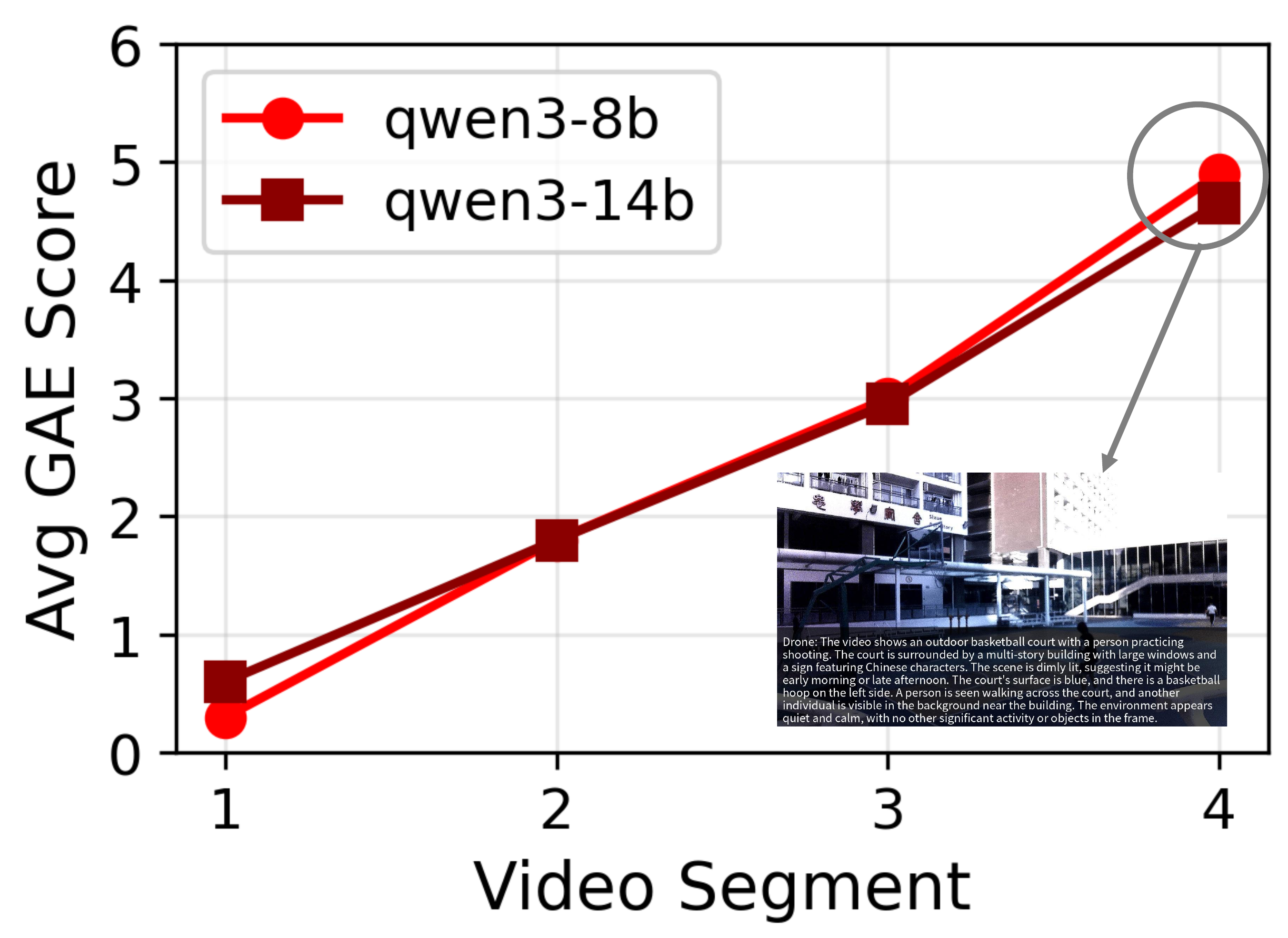}
		\caption{GAE score vs. segment.}
	\end{subfigure}
	\caption{GAE accuracy and score in real-world experiments.}
        \vspace{-0.1in}
	\label{fig:fig18}
\end{figure}

\begin{figure}[t]
    \centering
    \includegraphics[width=0.49\textwidth]{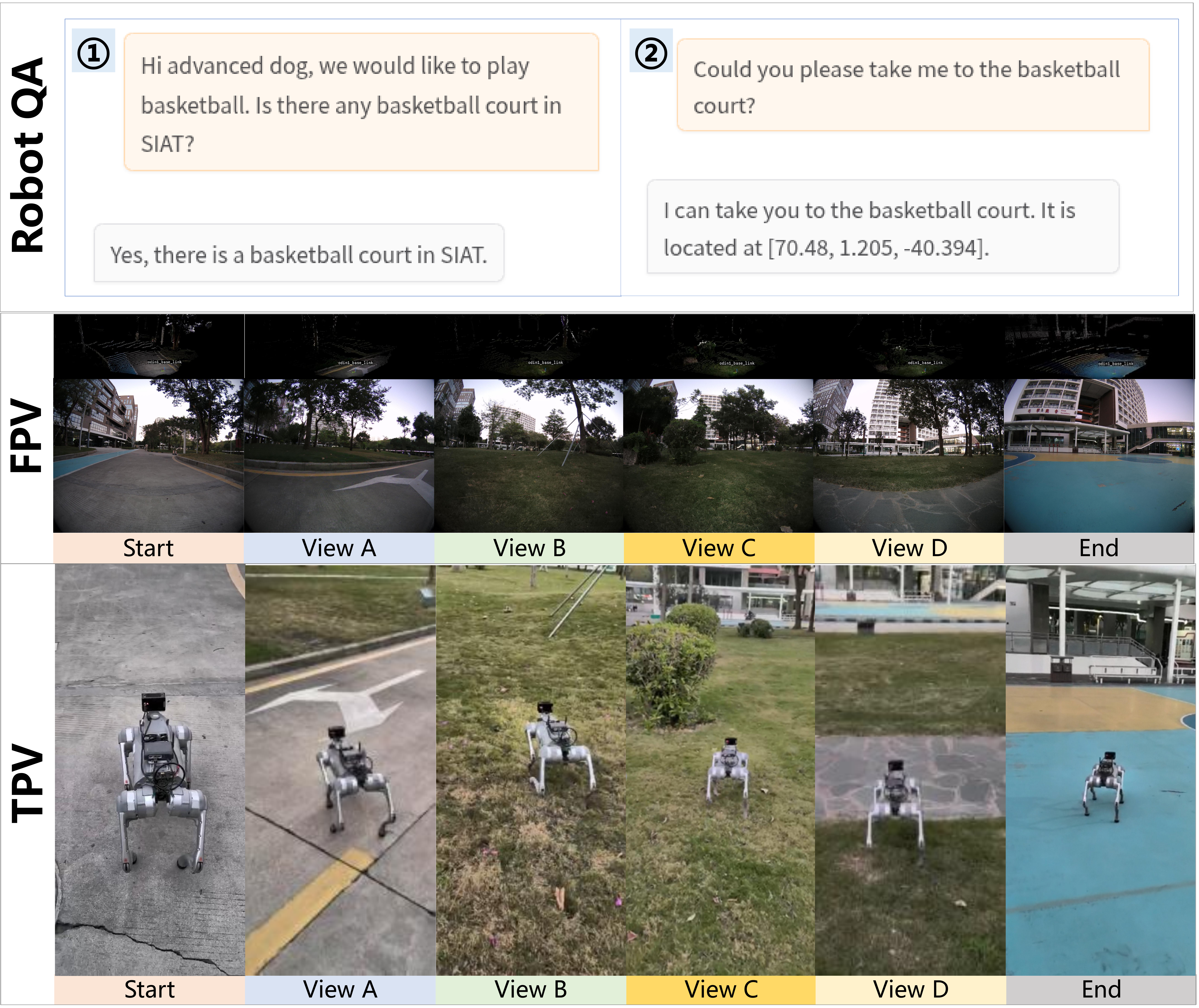}
    \caption{Robot-dog QA and navigation using the UAV memory. Dog positions are marked by yellow boxes in Fig.~\ref{fig:fig16}(c).}
    \label{fig:fig19}
    \vspace{-0.05in}
\end{figure}

Finally, we conduct campus experiments using the quadcopter and Unitree Go2 robotic dog shown in Fig.~\ref{fig:fig16}(a) and  Fig.~\ref{fig:fig16}(b). 
Our drone is equipped with four cameras, a 3D lidar, and an onboard NUC computer. 
\texttt{FAST-LIO2} is used to obtain real-time centimeter-level localization of the drone, with the recorded trajectory shown as a green line in the point cloud map of Fig.~\ref{fig:fig16}(c). 
The front camera collects a total of $1207$ images. 
We divide the entire video into $5$ equal-length segments $[0,1,2,3,4]$. 
{The start and end frames of all segments are shown in Fig.~\ref{fig:fig17}, including both first-person views (FPVs) and third-person views (TPVs).}
The images are captioned using VLM \texttt{Qwen3-VL-8B}. 
As shown in Fig.~\ref{fig:fig17}, the drone recognizes the campus park, identifies a statue of a person holding a sword, and locates a residential building.

We adopt segment $0$ as initial memory $\mathcal{M}_0$. 
The server can collect only one of segments $1$--$4$.
We use the proposed GAE model to measure the memory qualities of segments $1$--$4$. 
Fig.~\ref{fig:fig18}(a) shows high GAE QA accuracy for segment 1. 
This agrees with the Start, A, and B views in Fig.~\ref{fig:fig17}, which show similar scenes in segments 1 and 0. In contrast, the GAE score of segment 4 is the highest as shown in Fig.~\ref{fig:fig18}(b). 
This result demonstrates that GAE identifies fresh memory in real-world data. 
Based on the GAE scores, the LAQA system fetches segment 4 for memory aggregation.

After aggregating the memories, we use natural-language prompts to guide the robotic dog. 
As shown in Fig.~\ref{fig:fig19}, the user wishes to play sports and asks the robot whether there is a basketball court. 
Upon receiving confirmation from the dog, the user can further prompt: ``\texttt{Could you take me to the basketball court?}''
The robot feeds this prompt, together with the drone's memory, into the \texttt{Qwen3-8B} model for agentic inference, to retrieve position information from the recorded views. 
If a position is retrieved, it is transformed from the drone coordinate system to the dog coordinate system and passed to the path planner and locomotion controller to generate actions. 
The task succeeds when the robot is within a predefined distance of the target venue.
Here, the final output of agentic inference is given by: ``
\texttt{I can take you to the basketball court. It is located at [70.48, 1.205, -40.394]}''.
Using the goal position generated by the LLM, the dog follows the path shown in Fig.~\ref{fig:fig19} and successfully arrives at the basketball court.
The complete dog trajectory is shown as a red line in Fig.~\ref{fig:fig16}(c).
This result suggests that LAQA enhances ground embodied navigation performance in complex scenarios.

\section{Conclusion}\label{section7}

This paper introduced a task-driven interface between distributed
UAV memory and wireless resource allocation. GAE estimates the marginal value
of a candidate through the deployed captioning, retrieval, and reasoning
pipeline. MemCen combines that value with heterogeneous payload and channel
constraints. Its QoM-aware capped water-filling law shows analytically how task
value changes physical-layer priority and saturation. PSO and L2M provide
practical optimization and low-latency inference.
Town04, dynamic Town05, and real PMAS results support the utility,
robustness, and real-data feasibility of MemCen. Future work will extend the
method to larger aerial networks and broader platform and task
distributions.

\bibliographystyle{IEEEtran}
\bibliography{ref}

@inproceedings{li2026memory,
  title={Memory centric power allocation for multi-agent embodied question answering},
  author={C. Li and S. Wang and K. Ye and W. Yuan and B. Zhou and Y.-C. Wu and C. Xu and H. Arslan},
  booktitle={Proc. GLOBECOM},
  year={2026}
}

@article{wen2023task,
  title={Task-oriented sensing, computation, and communication integration for multi-device edge {AI}},
  author={Wen, Dingzhu and Liu, Peixi and Zhu, Guangxu and Shi, Yuanming and Xu, Jie and Eldar, Yonina C and Cui, Shuguang},
  journal={IEEE Trans. Wireless Commun.},
  volume={23},
  number={3},
  pages={2486--2502},
  year={2024}
}

@article{shi2023task,
  title={Task-oriented communications for {6G}: Vision, principles, and technologies},
  author={Shi, Yuanming and Zhou, Yong and Wen, Dingzhu and Wu, Youlong and Jiang, Chunxiao and Letaief, Khaled B},
  journal={IEEE Wireless Commun.},
  volume={30},
  number={3},
  pages={78--85},
  year={2023},
  publisher={IEEE}
}

@article{yan2022resource,
  title={Resource allocation for text semantic communications},
  author={Yan, Lei and Qin, Zhijin and Zhang, Rui and Li, Yongzhao and Li, Geoffrey Ye},
  journal={IEEE Wireless Commun. Lett.},
  volume={11},
  number={7},
  pages={1394--1398},
  year={2022}
}

@article{liu2025intelligent,
  title={Intelligent Semantic Communication Scheme Integrating {ISAC} for Low-Altitude Intelligent Networks},
  author={Liu, Shuai and Yang, Helin and Xie, Wancheng and Zheng, Mengting},
  journal={IEEE Trans. Commun.},
  volume={74},
  pages={3018--3033},
  year={2026}
}

@article{tusha2024interference,
  title={Interference burden in wireless communications: A comprehensive survey from {PHY} layer perspective},
  author={Tusha, Armed and Arslan, H{\"u}seyin},
  journal={IEEE Commun. Surv. Tutor.},
  volume={27},
  number={4},
  pages={2204--2246},
  year={2025},
  publisher={IEEE}
}

@article{ye2025integrated,
  title={Integrated sensing and communications for low-altitude economy: A deep reinforcement learning approach},
  author={Ye, Xiaowen and Mao, Yuyi and Yu, Xianghao and Sun, Shu and Fu, Liqun and Xu, Jie},
  journal={IEEE Trans. Wireless Commun.},
  year={2026},
  pages={351--367},
  volume={25}
}

@article{jiang2025integrated,
  title={Integrated sensing and communication for low altitude economy: Opportunities and challenges},
  author={Jiang, Yihang and Li, Xiaoyang and Zhu, Guangxu and Li, Hang and Deng, Jing and Han, Kaifeng and Shen, Chao and Shi, Qingjiang and Zhang, Rui},
  journal={IEEE Commun. Mag.},
  volume={63},
  number={12},
  pages={72--78},
  year={2025}
}

@article{wang2026low,
  author       = {Wang, B. and Kang, H. and Li, J. and Sun, G. and Sun, Z. and Wang, J. and Niyato, D. and Mao, S.},
  title        = {Low-Altitude Satellite-{AAV} Collaborative Joint Mobile Edge Computing and Data Collection via Diffusion-based Deep Reinforcement Learning},
  journal      = {IEEE Trans. Mob. Comput.},
  year         = {2026},
  day          = {16}
}

@article{wang2025llm,
  author  = {Wang, Jiawei and Tian, Yang and Li, Junjie and Sun, Haofeng and Tian, Hui and Zhang, Ping},
  title   = {{LLM}-Empowered Semantic Communication for Multi-Task {3D} Scene Understanding in Low-Altitude Economy Networks},
  journal = {IEEE Trans. Cog. Commun. Netw.},
  year    = {2025},
  volume  = {12},
  pages   = {4896--4910}
}

@article{liu2025goal,
  title={Goal-oriented semantic communication for wireless visual question answering},
  author={Liu, Sige and Li, Nan and Deng, Yansha and Quek, Tony QS},
  journal={IEEE J. Sel. Areas Commun.},
  volume={43},
  number={12},
  pages={4247--4261},
  year={2025}
}

@article{zhang2023multistep,
  author  = {Zhang, Meimei and Chen, Fang and Li, Bin},
  title   = {Multistep Question-Driven Visual Question Answering for Remote Sensing},
  journal = {IEEE Trans. Geosci. Remote Sens.},
  year    = {2023},
  volume  = {61},
  pages   = {1--12}
}

@article{jin2025co,
  title={Co-Design of Sensing, Communications, and Control for Low-Altitude Wireless Networks},
  author={Jin, Haijia and Wu, Jun and Yuan, Weijie and Liu, Fan and Cui, Yuanhao},
  journal={IEEE Trans. Mob. Comput.},
  volume={24},
  number={11},
  pages={12035--12048},
  year={2025},
}

@article{li2024survey,
  title={A survey on deep active learning: Recent advances and new frontiers},
  author={Li, Dongyuan and Wang, Zhen and Chen, Yankai and Jiang, Renhe and Ding, Weiping and Okumura, Manabu},
  journal={IEEE Trans. Neural Netw. Learn. Syst.},
  volume={36},
  number={4},
  pages={5879--5899},
  year={2025},
  publisher={IEEE}
}

@article{cheng2025development,
  title={Development and Application of Coverage Control Algorithms: A Concise Review},
  author={Cheng, Bin and He, Mingyuan and Zhu, Zhongpan and He, Bin and Chen, Jie},
  journal={IEEE Trans. Autom. Sci. Eng.},
  volume={22},
  pages={14906--14927},
  year={2025}
}

@article{sun2016majorization,
  title={Majorization-minimization algorithms in signal processing, communications, and machine learning},
  author={Sun, Ying and Babu, Prabhu and Palomar, Daniel P},
  journal={IEEE Trans. Signal Process.},
  volume={65},
  number={3},
  pages={794--816},
  year={2017}
}

@article{shlezinger2023model,
  title={Model-based deep learning},
  author={Shlezinger, Nir and Whang, Jay and Eldar, Yonina C and Dimakis, Alexandros G},
  journal={Proc. IEEE},
  volume={111},
  number={5},
  pages={465--499},
  year={2023}
}

@article{Liu2024survey,
  title={{A survey of recent advances in optimization methods for wireless communications}},
  author={Liu, Ya-Feng and Chang, Tsung-Hui and Hong, Mingyi and Wu, Zheyu and So, Anthony Man-Cho and Jorswieck, Eduard A. and Yu, Wei},
  journal={IEEE J. Sel. Areas Commun.},
  volume={42},
  number={11},
  pages={2992--3031},
  year={2024}
}

@article{liu2024coverage,
  title={A coverage-aware task allocation method for {UAV}-assisted mobile crowd sensing},
  author={Liu, Xinbin and Wang, Ye and Gao, Hui and Ngai, Edith CH and Zhang, Bo and Wang, Chuhan and Wang, Wendong},
  journal={IEEE Trans. Veh. Technol.},
  volume={73},
  number={7},
  pages={10642--10654},
  year={2024},
  publisher={IEEE}
}

@inproceedings{sermanet2024robovqa,
  title={{RoboVQA}: Multimodal long-horizon reasoning for robotics},
  author={Sermanet, Pierre and Ding, Tianli and Zhao, Jeffrey and Xia, Fei and Dwibedi, Debidatta and Gopalakrishnan, Keerthana and Chan, Christine and Dulac-Arnold, Gabriel and Maddineni, Sharath and Joshi, Nikhil J and others},
  booktitle={Proc. ICRA},
  pages={645--652},
  year={2024}
}

@inproceedings{xu2024mobility,
  title={{Mobility VLA}: Multimodal instruction navigation with long-context {VLMs} and topological graphs},
  author={Xu, Zhuo and Chiang, Hao-Tien Lewis and Fu, Zipeng and Jacob, Mithun George and Zhang, Tingnan and Lee, Tsang-Wei Edward and Yu, Wenhao and Schenck, Connor and Rendleman, David and Shah, Dhruv and others},
  booktitle={Proc. CoRL},
  year={2024}
}

@inproceedings{das2018embodied,
  title={Embodied question answering},
  author={Das, Abhishek and Datta, Samyak and Gkioxari, Georgia and Lee, Stefan and Parikh, Devi and Batra, Dhruv},
  booktitle={Proc. CVPR},
  pages={1--10},
  year={2018}
}

@inproceedings{anwar2025remembr,
  title={{ReMEmbR}: Building and reasoning over long-horizon spatio-temporal memory for robot navigation},
  author={Anwar, Abrar and Welsh, John and Biswas, Joydeep and Pouya, Soha and Chang, Yan},
  booktitle={Proc. ICRA},
  pages={2838--2845},
  year={2025}
}

@inproceedings{carla,
  title={{CARLA}: An open urban driving simulator},
  author={Dosovitskiy, Alexey and Ros, German and Codevilla, Felipe and Lopez, Antonio and Koltun, Vladlen},
  booktitle={CoRL},
  pages={1--16},
  year={2017}
}

@article{zheng2016wireless,
  title={Wireless max--min utility fairness with general monotonic constraints by {Perron--Frobenius} theory},
  author={Zheng, Liang and Hong, Y-W Peter and Tan, Chee Wei and Hsieh, Cheng-Lin and Lee, Chia-Han},
  journal={IEEE Trans. Inf. Theory},
  volume={62},
  number={12},
  pages={7283--7298},
  year={2016}
}

@article{wang2020machine,
  title={Machine intelligence at the edge with learning centric power allocation},
  author={Wang, Shuai and Wu, Yik-Chung and Xia, Minghua and Wang, Rui and Poor, H Vincent},
  journal={IEEE Trans. Wireless Commun.},
  volume={19},
  number={11},
  pages={7293--7308},
  year={2020},
  publisher={IEEE}
}

@article{rinaldi2009new,
  title={New results on the equivalence between zero-one programming and continuous concave programming},
  author={Rinaldi, Francesco},
  journal={Optim. Lett.},
  volume={3},
  pages={377--386},
  year={2009}
}

@article{lucidi2010exact,
  title={Exact Penalty Functions for Nonlinear Integer Programming Problems},
  author={Lucidi, Stefano and Rinaldi, Francesco},
  journal={J. Optim. Theory Appl.},
  volume={145},
  pages={479--488},
  year={2010}
}

@article{wang2020angle,
  title={Angle aware user cooperation for secure massive {MIMO} in {Rician} fading channel},
  author={Wang, Shuai and Wen, Miaowen and Xia, Minghua and Wang, Rui and Hao, Qi and Wu, Yik-Chung},
  journal={IEEE J. Sel. Areas Commun.},
  volume={38},
  number={9},
  pages={2182--2196},
  year={2020}
}

@article{wang2022edge,
author    = {Shuai Wang and Yuncong Hong and Rui Wang and Qi Hao and Yik-Chung Wu and Derrick Wing Kwan Ng},
title     = {Edge Federated Learning via Unit-Modulus Over-The-Air Computation},
journal   = {IEEE Trans. Commun.},
volume    = {70},
number    = {5},
pages     = {3141--3156},
year      = {2022}
}

\end{document}